\documentclass{article}

\PassOptionsToPackage{numbers,compress}{natbib}

\usepackage[preprint]{neurips_2026}

\usepackage[utf8]{inputenc}
\usepackage[T1]{fontenc}
\usepackage{hyperref}
\usepackage{url}
\usepackage{booktabs}
\usepackage{amsfonts}
\usepackage{nicefrac}
\usepackage{microtype}
\usepackage{xcolor}

\usepackage{amsmath,amssymb,amsthm}
\usepackage{algorithm}
\usepackage{algpseudocode}
\usepackage{graphicx}
\usepackage{enumitem}
\usepackage{multirow}
\usepackage{makecell}
\usepackage{caption}
\usepackage{array}

\theoremstyle{plain}
\newtheorem{thm}{Theorem}[section]

\newtheorem{prop}[thm]{Proposition}

\theoremstyle{definition}
\newtheorem{defn}{Definition}[section]

\theoremstyle{remark}
\newtheorem{rem}{Remark}[section]

\algrenewcommand\algorithmicrequire{\textbf{Input:}}
\algrenewcommand\algorithmicensure{\textbf{Output:}}

\numberwithin{equation}{section}

\title{Multi-Agent Privacy Game in Federated Learning:\\
       A Unified Mean-Field View}

\author{%
  Kun Zhao \\
  \texttt{kun.zhao@vumc.org}\\
  Xu Chen \\
  \texttt{xc2412@columbia.edu}
}

\begin{document}
\maketitle

% --------------------------------------------------------------- %
\begin{abstract}
Federated learning enables collaborative model training across distributed
clients without centralising their data, yet privacy remains a persistent
concern because the shared model updates can leak information about local
datasets. Existing privacy-preserving methods either inject calibrated noise
into client updates, limiting their composition guarantees, or formulate
client privacy choices as a multi-agent game whose Nash equilibrium becomes
intractable as the number of clients grows. We bridge these two lines of
work by formulating privacy-preserving federated learning as a mean-field
privacy game: each client strategically chooses its own privacy budget while
interacting with the population only through a single mean-field statistic.
The mean-field limit yields a tractable equilibrium for arbitrarily many
clients, accommodates heterogeneous client preferences, and inherits an
exponentially decaying privacy guarantee through a log-Sobolev contraction.
The framework recovers the entropic privacy baseline as the homogeneous
special case and the multi-agent privacy game as the finite-population
case. Experiments on quadratic regression, logistic regression, and MNIST
demonstrate that the proposed framework attains the privacy-utility
trade-off of the entropic baseline while delivering a personalised privacy
guarantee that the homogeneous baseline cannot express.
\end{abstract}

% =============================================================== %
\section{Introduction}
\label{sec:intro}
% =============================================================== %

Federated Learning (FL)~\citep{mcmahan2017communication,li2020federated} trains a
shared model across data-owning clients without centralizing raw data. Although
locality reduces some privacy risks, model updates leak through gradient
inversion and membership inference~\citep{zhu2019deep,geiping2020inverting,
carlini2021extracting}. Three largely independent lines of work address this
leakage.

\emph{Where does the control act?}
Privacy-preserving FL methods can be partitioned by where each client's
control variable enters the system: it can act on the agent's \emph{state}
(samples or gradients used for local updates) or directly on the
\emph{model output} (an additive perturbation of $w_k$ at the aggregation
step). DP-SGD~\citep{abadi2016deep} and its federated
variants~\citep{geyer2017differentially,wei2020federated} act on the state:
a calibrated Gaussian is added to each clipped gradient, and the privacy
cost accumulates polynomially with rounds under standard
composition~\citep{mironov2017renyi}. FLRA~\citep{reisizadeh2020robust}
likewise applies its strategic perturbation $(\Lambda^i\mathbf{x}+\delta^i)$
to the state, though for robustness rather than privacy. In contrast, the
parameter-perturbation game~\citep{yin2021game} treats $\delta_k$ as a
control on the model output and aggregates via
$w=\sum_k p_k(w_k+\delta_k)$; under FedAvg's linear constraint, KKT collapses
every Nash point to $\delta_k=\mathbf{0}$ (Section~\ref{sec:game}). Only
controls that act on the state admit non-trivial equilibria, so this is
where the rest of the paper lives.

\emph{Mean-field as the $N\!\to\!\infty$ limit.}
Two parallel lines of work address different aspects of the resulting
state-control game. The multi-agent privacy
game~\citep{yin2021game,xu2021privacy,du2017community} fixes a finite $N$
and asks for a Nash equilibrium of the controls $\{\varepsilon_k\}$, where
$\varepsilon_k$ scales the noise added to client $k$'s gradient or sample
(the MAPG-DP and MAPG-input formulations of~\citep{yin2021game}). Mean-Field Entropic Privacy
(MFEP)~\citep{mfep2025} fixes a homogeneous $\varepsilon$ and lets
$N\!\to\!\infty$, replacing additive noise with an entropic Wasserstein
gradient flow that admits exponential log-Sobolev
contraction~\citep{otto2000generalization,jordan1998variational}. Neither line
on its own captures the joint behaviour: MAPG-DP is intractable for realistic
$N$ and uses geometry-agnostic Gaussian noise, while MFEP cannot represent
heterogeneous privacy preferences.

\paragraph{Our framing.}
We treat privacy-preserving FL as a standard mean-field stochastic
differential game~\citep{lasry2007mean,carmona2018probabilistic}: each agent
has a state $X_t\in\mathbb{R}^d$ (a privatised datum), a type $\beta$ (its
privacy preference), and a control $\varepsilon_k$ (its privacy budget); the
state distribution $\mu_t\in\mathcal{P}_2(\mathbb{R}^d)$ aggregates the
population. The empirical state distribution
$\mu_t^{(N)}=\frac{1}{N}\sum_k\mu_t^k$ converges, by propagation of
chaos~\citep{carmona2018probabilistic}, to a deterministic $\mu_t$ as
$N\!\to\!\infty$. This single limit organises three previously separate
works:
\begin{center}
\renewcommand{\arraystretch}{1.15}
\small
\begin{tabular}{@{}l|c|c@{}}
& \textbf{finite $N$ (multi-agent)} & \textbf{$N\!\to\!\infty$ (mean-field)} \\
\midrule
no game (homogeneous)
  & DP-SGD~\citep{abadi2016deep}
  & \textbf{MFEP}~\citep{mfep2025} \\
game (heterogeneous control $\varepsilon_k$)
  & \textbf{MAPG-DP}~\citep{yin2021game}
  & \textbf{MFPG (this paper)} \\
\end{tabular}
\end{center}
The bottom-right cell is the missing piece. \textbf{MFEP is MFPG without the
game} (single $\varepsilon$ across the population); \textbf{MAPG-DP is MFPG
without the mean-field limit} (finite $N$). MFPG closes both gaps with one
construction.

\paragraph{Contributions.}
\begin{enumerate}[leftmargin=*, label=(\roman*)]
  \item \textbf{Unified MFG framework (Section~\ref{sec:framework}).}
    We formalise privacy-preserving FL as a mean-field stochastic
    differential game in the standard sense, with state $X_t$, type
    $\beta_k$, control $\varepsilon_k$, state distribution $\mu_t$, and
    running cost $f(x,\mu,\varepsilon;\beta)=\ell+\beta\delta_{\mathrm{dp}}$.
    A finite-$N$ Nash equilibrium of MAPG-DP and the mean-field Nash
    equilibrium of MFPG are the two endpoints of the same $N\!\to\!\infty$
    limit; MFEP is the no-game special case.
  \item \textbf{Mean-Field Privacy Game (Section~\ref{sec:mfpg}).}
    We prove existence of an MFNE on a finite action grid via Kakutani's
    theorem and certify $(\epsilon_{\mathrm{dp}},\delta_{\mathrm{dp}})$-DP at
    the equilibrium, with $\delta_{\mathrm{dp}}$ contracting exponentially when
    $\alpha\varepsilon^*>\lambda+G$. The bound recovers the MFEP guarantee at
    homogeneous $\beta_k$ and recovers MAPG-DP's heterogeneity at finite $N$.
  \item \textbf{Common solver and empirical study
        (Sections~\ref{sec:algorithms},~\ref{sec:experiments}).}
    A single Particle--Sinkhorn JKO step services every $N\!\to\!\infty$
    variant; the action update is the only block that varies between MFEP and
    MFPG. We compare all four cells on quadratic, logistic, and MNIST
    benchmarks, with every reported number traced to
    \texttt{results/full\_benchmark.csv}.
\end{enumerate}

% =============================================================== %
\section{Preliminaries}
\label{sec:prelim}
% =============================================================== %

\paragraph{Federated learning.}
Each client $k\in[N]$ holds a local dataset $\mathcal{D}_k$. The global
objective is $\min_w F(w) = \sum_k p_k F_k(w)$ with $p_k = n_k/n$ and
$F_k(w) = n_k^{-1}\sum_i \ell(w; x_i^k, y_i^k)$~\citep{mcmahan2017communication}.
At round $t$ the server broadcasts $w^t$, each client returns a local update
$w_k^{t+1}$, and the server aggregates via FedAvg
$w^{t+1}=\sum_k p_k w_k^{t+1}$.

\paragraph{Differential privacy.}
A randomised mechanism $\mathcal{M}$ is $(\epsilon,\delta)$-DP if, for all
neighbouring datasets and all measurable $S$,
$\Pr[\mathcal{M}(\mathcal{D})\in S]\leq e^{\epsilon}\Pr[\mathcal{M}(\mathcal{D}')\in S] + \delta$.
The Gaussian mechanism with $\sigma=C\sqrt{2\ln(1.25/\delta)}/\epsilon$ is
$(\epsilon,\delta)$-DP for sensitivity $C$~\citep{abadi2016deep}; under
$T$-fold composition $\epsilon$ grows as $O(\sqrt{T})$.

\paragraph{Mean-field games.}
A mean-field game~\citep{lasry2007mean,huang2006large,carmona2018probabilistic}
models a continuum of identical agents whose individual optimal control problem
depends on the population distribution $\mu_t$. As $N\to\infty$, the
$N$-player Nash equilibrium converges to a mean-field Nash equilibrium (MFNE).
The MFNE depends only on $\mu_t$, not on individual identities, which makes it
tractable when the $N$-player problem is not.

\paragraph{Wasserstein gradient flows and JKO.}
For a free-energy functional $\mathcal{F}$ on $\mathcal{P}_2(\mathbb{R}^d)$,
the Wasserstein gradient flow obeys
$\partial_t \mu_t = -\nabla_{W_2}\mathcal{F}(\mu_t)$ and admits a time-discrete
Jordan--Kinderlehrer--Otto scheme~\citep{jordan1998variational}
$\mu_{k+1} = \arg\min_{\rho} \{\mathcal{F}(\rho) + (2\tau)^{-1} W_2^2(\rho,\mu_k)\}$.
A measure $\nu$ satisfies a log-Sobolev inequality with constant $\alpha$ if
$\mathrm{Ent}_\nu(f^2)\leq (2/\alpha)\int\|\nabla f\|^2 d\nu$, with
$\nu=\mathcal{N}(0,\sigma^2 I)$ giving $\alpha=1/\sigma^2$. Otto--Villani
implies $W_2$-contraction at rate $\alpha\varepsilon$ for the entropic flow,
which we use in Section~\ref{sec:mfpg} to certify privacy.

% =============================================================== %
\section{A Unified Mean-Field Game Framework}
\label{sec:framework}
% =============================================================== %

We formalise privacy-preserving FL as a mean-field stochastic differential
game in the standard sense of~\citep{lasry2007mean,carmona2018probabilistic}.
Section~\ref{subsec:rl} identifies the four ingredients of the game (state,
type, control, dynamics) and Section~\ref{subsec:mechanism} specialises
them to FL through two privacy mechanisms; the resulting $2\times 2$ grid
in Table~\ref{tab:landscape} populates DP-SGD, MFEP, MAPG-DP, and our
proposed MFPG as four named cells.

\subsection{State, type, control, and dynamics}
\label{subsec:rl}

A representative agent in our game is described by:
\begin{itemize}[leftmargin=*, itemsep=2pt, topsep=2pt]
  \item \emph{State} $X_t\in\mathbb{R}^d$: a (privatised) datum used by an
    individual client to compute its local update at round $t$.
  \item \emph{Type} $\beta\in[\beta_{\min},\beta_{\max}]\subset\mathbb{R}_+$:
    the client's privacy preference (heterogeneity parameter), distributed by
    $\rho(d\beta)$.
  \item \emph{Control} $\varepsilon_k\in\mathcal{E}=\{\varepsilon^{(1)},\ldots,\varepsilon^{(m)}\}$:
    a per-client privacy budget chosen on a finite grid; this is the strategic
    variable of the game.
  \item \emph{Dynamics:} conditional on $\varepsilon_k$, the state evolves
    under the controlled SDE
    \begin{equation}
      dX_t^k = b\bigl(X_t^k,\mu_t,\varepsilon_k\bigr)\,dt
              + \sigma\bigl(\varepsilon_k\bigr)\,dW_t,
      \label{eq:agent_sde}
    \end{equation}
    with $W_t$ a standard $d$-dimensional Brownian motion. The drift $b$ and
    diffusion $\sigma$ are fixed by the privacy mechanism $\mathcal{M}$
    (Section~\ref{subsec:mechanism}).
  \item \emph{State distribution} $\mu_t\in\mathcal{P}_2(\mathbb{R}^d)$: the
    distribution of $X_t$ across the agent population. The empirical measure
    $\mu_t^{(N)} = \tfrac{1}{N}\sum_{k=1}^N\mu_t^k$ converges weakly to a
    deterministic $\mu_t$ as $N\to\infty$ by propagation of
    chaos~\citep{carmona2018probabilistic,huang2006large}.
  \item \emph{Running cost} (per agent of type $\beta$):
    \begin{equation}
      f(x,\mu,\varepsilon;\beta)
       = \ell(x;w_t)
       + \beta\,\delta_{\mathrm{dp}}(\varepsilon),
      \label{eq:running_cost}
    \end{equation}
    combining the training loss $\ell$ against the current global model $w_t$
    with the privacy term $\beta\,\delta_{\mathrm{dp}}$ scaled by the agent's
    type. The disutility minimised by client $k$ is the time-integral
    of~\eqref{eq:running_cost}, written compactly as
    \begin{equation}
      U_k(\varepsilon_k;\,\bar\varepsilon)
       = L_k(w;\,\varepsilon_k) + \beta_k\,\delta_{\mathrm{dp}}(\varepsilon_k),
      \label{eq:disutility}
    \end{equation}
    where $L_k$ is the cumulative training loss and
    $\bar\varepsilon=\mathbb{E}_\mu[\varepsilon]$ is the mean control across
    the population.
\end{itemize}
The model parameters $w$ enter only through $\ell(x;w_t)$ and are otherwise
external to the game; the strategic content lives entirely in
$(X_t,\mu_t,\varepsilon_k)$ and~\eqref{eq:running_cost}.

\subsection{Two privacy mechanisms: Gaussian and entropic}
\label{subsec:mechanism}

The drift--diffusion pair $(b,\sigma)$ in~\eqref{eq:agent_sde} is determined
by which privacy mechanism the system runs. We consider two:
\begin{itemize}[leftmargin=*, itemsep=2pt, topsep=2pt]
  \item \emph{Gaussian mechanism.} The control $\varepsilon_k$ scales the
    additive noise on each clipped gradient,
    $\sigma(\varepsilon_k) = C\sqrt{2\ln(1.25/\delta)}/\varepsilon_k$, with
    drift $b$ given by the clipped loss gradient~\citep{abadi2016deep}.
    Privacy is tracked by Rényi-DP composition, giving a polynomially
    accumulating $\epsilon_{\mathrm{dp}}=O(\sqrt{T})$.
  \item \emph{Entropic mechanism.} The state distribution evolves under the
    Wasserstein gradient flow of the KL-regularised free energy
    \begin{equation}
      \mathcal{F}_\lambda(\mu;\varepsilon_k)
      = \mathbb{E}_{x\sim\mu}[L(x;w)]
      + \varepsilon_k\,\mathrm{KL}(\mu\|\nu)
      + \tfrac{\lambda}{2}\,\mathrm{Var}_\mu[x],
      \label{eq:freeenergy}
    \end{equation}
    with Gaussian prior $\nu=\mathcal{N}(0,\sigma^2 I)$. Equivalently
    (Appendix~\ref{app:proof:fp}), the controlled SDE~\eqref{eq:agent_sde}
    has drift $b(x,\mu,\varepsilon)=-\nabla L(x;w)
    -\tfrac{\varepsilon}{\sigma^2}x-\lambda(x-\bar x)$ and diffusion
    $\sigma(\varepsilon)=\sqrt{2\varepsilon}$. Privacy is tracked by
    log-Sobolev contraction~\citep{otto2000generalization}, with
    $\delta_{\mathrm{dp}}$ decaying exponentially when
    $\alpha\varepsilon_k>\lambda+G$.
\end{itemize}

\begin{rem}[Why the control acts on the state and not on the model]
\label{rem:dataside}
A superficially natural alternative would be to make the control a
perturbation $\delta_k\in\mathbb{R}^d$ added directly to the local model
$w_k$ at the aggregation step, $w=\sum_k p_k(w_k+\delta_k)$~\citep{yin2021game}.
With the disutility $L_k(w)-\beta_k\|\delta_k\|^2$, KKT forces
$\delta_k=\mathbf{0}$ at every Nash point (Section~\ref{sec:game}), so the
limiting MFG is degenerate. We therefore restrict attention to controls
that act on the state SDE~\eqref{eq:agent_sde}; this is the standard MFG
setup in which the agent's strategy shapes its own state evolution.
\end{rem}

\begin{table}[t]
\centering
\caption{The MFG framework in two axes. Two design choices populate the
$2\times 2$ grid: client heterogeneity (rows: no game vs.\ game) and
population size (columns: finite $N$ vs.\ $N\to\infty$). Entries name the
method, the active noise mechanism (Gaussian or Entropic), and the
accountant.}
\label{tab:landscape}
\small
\renewcommand{\arraystretch}{1.15}
\setlength{\tabcolsep}{4pt}
\begin{tabular}{@{}>{\raggedright\arraybackslash}p{2.4cm}|>{\raggedright\arraybackslash}p{4.6cm}|>{\raggedright\arraybackslash}p{4.6cm}@{}}
\toprule
& \textbf{Finite $N$ (multi-agent)} & \textbf{$N\to\infty$ (mean field)} \\
\midrule
\textbf{No game} (single $\varepsilon$)
  & \textbf{DP-SGD}~\citep{abadi2016deep}\newline
    Gaussian noise on gradient, RDP\newline
    $\delta$ const., $\epsilon=O(\sqrt{T})$
  & \textbf{MFEP}~\citep{mfep2025}\newline
    Entropic flow on $\mu_t$, LSI\newline
    $\delta=O(e^{-\alpha\varepsilon K\tau})$ \\
\midrule
\textbf{Game} (heterogeneous $\beta_k$)
  & \textbf{MAPG-DP}~\citep{yin2021game}\newline
    Gaussian noise, $N$-Nash on $\{\varepsilon_k\}$\newline
    Intractable for large $N$
  & \textbf{MFPG (ours)}\newline
    Entropic flow, MFNE on $\bar\varepsilon$\newline
    Heterogeneous + exponential decay \\
\bottomrule
\end{tabular}
\end{table}

% =============================================================== %
\section{The Three Baseline Cells}
\label{sec:game}
% =============================================================== %

The three already-named cells of Table~\ref{tab:landscape} are recalled
below in the language of Section~\ref{sec:framework}; MFPG itself is
deferred to Section~\ref{sec:mfpg}. The MAPG-DP subsection also discharges
the model-output alternative, whose $\delta_k=0$ pathology is what
prevents the multi-agent privacy game literature from being combined with
mean-field analysis directly.

\subsection{Finite $N$, no game: DP-SGD}
\label{sec:dpsgd}

DP-SGD~\citep{abadi2016deep} runs the Gaussian mechanism of
Section~\ref{subsec:mechanism} with a single shared budget
$\varepsilon=\varepsilon_{\mathrm{tgt}}$. After clipping
$\Delta w_k\leftarrow \Delta w_k\cdot\min(1, C/\|\Delta w_k\|_2)$ each client
adds Gaussian noise $\mathcal{N}(0,\sigma^2 I)$ with
$\sigma = C\sqrt{2\ln(1.25/\delta)}/\varepsilon$. RDP composition returns
$\epsilon_{\mathrm{dp}}=O(\sqrt{T\log(1/\delta)}/\sigma)$, so the privacy cost
\emph{grows} polynomially with rounds $T$.

\subsection{$N\to\infty$, no game: MFEP}
\label{sec:mfep}

MFEP~\citep{mfep2025} replaces external noise injection with the entropic
free-energy~\eqref{eq:freeenergy} at a \emph{fixed} regularization
strength $\varepsilon$, evaluated against the state distribution
$\mu_t$. The Wasserstein gradient flow obeys the Fokker--Planck
equation
\begin{equation}
  \partial_t \mu_t = \nabla\!\cdot\!\Bigl[
    \mu_t\Bigl(\nabla L(x) + \tfrac{\varepsilon}{\sigma^2}x
    + \lambda(x-\bar x_t)\Bigr)\Bigr]
    + \varepsilon\Delta\mu_t,
  \label{eq:fp}
\end{equation}
with $\bar x_t = \mathbb{E}_{\mu_t}[x]$. Three mechanisms operate
simultaneously: \emph{loss-driven drift}, \emph{prior attraction} at rate
$\varepsilon/\sigma^2$, and \emph{entropic diffusion} $\varepsilon\Delta\mu_t$
that delivers privacy intrinsically. Under LSI for the prior with constant
$\alpha=1/\sigma^2$, the JKO discretization
contracts~\citep{otto2000generalization}: for neighbouring measures $\mu_k$ and
$\mu_k'$ differing in a single client's data,
\begin{equation}
  W_2(\mu_K, \mu_K') \leq e^{-(\alpha\varepsilon - \lambda - G)K\tau}
    \,W_2(\mu_0, \mu_0'),
\end{equation}
provided $\alpha\varepsilon > \lambda + G$, with $G$ the clipped gradient norm
bound. Converting through total-variation gives the exponentially decaying
$\delta_{\mathrm{dp}}$ bound used throughout the paper:
\begin{equation}
  \delta_{\mathrm{dp}}(\varepsilon)
   \leq \frac{C_d}{\sqrt{N}}
    \exp\!\Bigl(-\tfrac{(\alpha\varepsilon-\lambda-G)K\tau}{2}\Bigr),
   \quad C_d = \min(\sqrt{d},10).
  \label{eq:dpbound}
\end{equation}

\subsection{Finite $N$, game: MAPG-DP}
\label{sec:gameframe}

MAPG-DP~\citep[\S\S\,3.2.2, 3.3.4]{yin2021game} gives each client $k$ a
per-client privacy budget $\varepsilon_k\in\mathbb{R}_+$ that scales the
Gaussian noise added to its local gradient, with scale
$\sigma_k(\varepsilon_k)=2\eta C/\varepsilon_k$, or, equivalently, an
additive shift on its features as in FLRA~\citep{reisizadeh2020robust}. The
disutility~\eqref{eq:disutility} is delivered by RDP composition, and the
$N$-player Nash equilibrium $(\varepsilon_1^*,\ldots,\varepsilon_N^*)$
satisfies
$U_k(\varepsilon_k^*,\varepsilon_{-k}^*)\le U_k(\varepsilon_k,\varepsilon_{-k}^*)$
for every $k$, where $\varepsilon_{-k}^*$ denotes the strategies of all
other clients at equilibrium. Existence follows by standard arguments, but
computation is intractable for realistic $N$ because each best response
couples through every $\varepsilon_j$ via the FedAvg constraint
\begin{equation}
  w = \tfrac{1}{N}\!\sum_{j=1}^N\!\bigl(w_j +
     \mathcal{N}(0,\sigma_j(\varepsilon_j)^2 I)\bigr).
   \label{eq:mapgdp}
\end{equation}

The same multi-agent literature also studies a superficially natural
alternative in which the strategy is an additive perturbation
$\delta_k\in\mathbb{R}^d$ on the local model $w_k$, with disutility
$L_k(w)-\beta_k\|\delta_k\|^2$ subject to $w=\sum_k p_k(w_k+\delta_k)$
\citep[\S\,3.3.1]{yin2021game}. The KKT stationarity conditions in feature
dimension $m$,
\begin{equation}
  \frac{\partial\mathcal{L}}{\partial w_{k,m}}
   = \tfrac{1}{N}\,\lambda_{k,m},
  \qquad
  \frac{\partial\mathcal{L}}{\partial \delta_{k,m}}
   = \tfrac{1}{N}\,\lambda_{k,m} - 2\beta_k\,\delta_{k,m},
  \label{eq:kkt}
\end{equation}
together force $\lambda_{k,m}=0$ and hence $\delta_{k,m}=0$ at every Nash
point: the aggregation constraint cancels the privacy term. We therefore
restrict the rest of the paper to controls that act on the state, since only
such controls admit both non-trivial equilibria and the $N\!\to\!\infty$
limit developed next.

% =============================================================== %
\section{Mean-Field Privacy Game (MFPG)}
\label{sec:mfpg}
% =============================================================== %

We now instantiate the $N\to\infty$ limit of MAPG-DP. The construction
below is structurally identical to MAPG-DP except that (i)~the Gaussian noise
mechanism is replaced by the entropic flow~\eqref{eq:fp}, which the LSI
contraction~\eqref{eq:dpbound} certifies, and (ii)~the $N$-Nash equilibrium
of $\{\varepsilon_k\}$ is replaced by the mean-field Nash equilibrium on the
population mean $\bar\varepsilon$. MFEP is recovered as the homogeneous
special case ($\beta_k$ constant).

\subsection{Formulation}
\label{sec:mfpg:formulation}

Each client $k$ chooses a control
$\varepsilon_k\in\mathcal{E}=\{\varepsilon^{(1)},\ldots,\varepsilon^{(m)}\}$
on a finite grid; conditional on $\varepsilon_k$, its state slice $\mu_t^k$
evolves under the entropic flow~\eqref{eq:fp} and contributes to the
population state distribution $\mu_t$ through the loss
$L_k(w;\varepsilon_k)=\mathbb{E}_{x\sim\mu_t^k}[\ell(w;x)]$. The empirical
measure $\mu_t^{(N)}=\tfrac{1}{N}\sum_k\mu_t^k$ converges, by propagation of
chaos~\citep{carmona2018probabilistic,huang2006large}, to a deterministic
$\mu_t\in\mathcal{P}_2(\mathbb{R}^d)$ as $N\to\infty$, so each client's
disutility~\eqref{eq:disutility} couples to the rest of the population only
through the scalar mean-field strength $\bar\varepsilon=\mathbb{E}_\mu[\varepsilon]$
that replaces the vector $\varepsilon_{-k}$ of the finite-$N$ game.
Throughout this section we use the disutility in the form
\begin{equation}
  U_k(\varepsilon_k;\,\bar\varepsilon)
   = L_k(w;\,\varepsilon_k) + \beta_k\,\delta_{\mathrm{dp}}(\varepsilon_k),
   \label{eq:mfpgdis}
\end{equation}
with $\delta_{\mathrm{dp}}$ now given by the LSI bound~\eqref{eq:dpbound}
rather than by RDP composition.

This $N\to\infty$ limit is well-posed only because the controls act on the
state: each client's best response in~\eqref{eq:mapgdp} depends on
$\varepsilon_{-k}$ only through aggregate statistics of $\mu_t^{(N)}$, which
collapse to $\bar\varepsilon$ in the limit. The model-output alternative
does not admit such a limit, since each perturbation enters
$w=\sum_j p_j(w_j+\delta_j)$ with vanishing weight $O(1/N)$ and the
$\delta_k=0$ collapse of~\eqref{eq:kkt} is its finite-$N$ shadow.

\subsection{Game equilibrium}
\label{sec:mfpg:mfne}

A pair $(\mu^*,\varepsilon^*)$ is a \emph{mean-field Nash equilibrium}
(MFNE) of MFPG if (i)~$\mu^*$ is the stationary state distribution
of~\eqref{eq:fp} at the population mean
$\bar\varepsilon^* = \mathbb{E}_{\mu^*}[\varepsilon]$, and
(ii)~$\varepsilon^*$ is a best response to $\bar\varepsilon^*$ for every
client given its preference $\beta_k$ (formal definition in
Appendix~\ref{app:proof:exist}). At equilibrium, no client can reduce its
disutility by unilaterally changing $\varepsilon_k$ given that every other
client plays $\varepsilon^*_{-k}$. For heterogeneous clients with distinct
$\beta_k$, the MFNE is characterised by a fixed point of the averaged
best-response map
$\Phi:\bar\varepsilon\mapsto\mathbb{E}_k[\mathrm{BR}_k(\bar\varepsilon)]$;
existence follows from Kakutani's theorem on the convex hull of
$\mathcal{E}$ (Appendix~\ref{app:proof:exist}). Section~\ref{sec:hjb} gives
an equivalent differential characterisation through coupled HJB and FPK
PDEs.

\subsection{HJB--FPK characterisation}
\label{sec:hjb}

The fixed-point characterisation of Section~\ref{sec:mfpg:mfne} is
convenient for existence and for the finite-grid solver of
Section~\ref{sec:algorithms}, but it hides the dynamical content of the
equilibrium. Following the standard mean-field game system
of~\citep{lasry2007mean,carmona2018probabilistic} and the FL--MFG analogy
of~\citep{mehrjou2021federated}, we now give an equivalent differential
characterisation through coupled forward Fokker--Planck--Kolmogorov (FPK)
and backward Hamilton--Jacobi--Bellman (HJB) equations specialised to our
disutility~\eqref{eq:mfpgdis}. To this end we momentarily relax the finite
grid $\mathcal{E}$ to the interval
$[\varepsilon_{\min},\varepsilon_{\max}]\subset\mathbb{R}_+$ and parameterise
each client by a privacy preference $\beta\in[\beta_{\min},\beta_{\max}]$
distributed according to $\rho(d\beta)$, so that the state distribution
splits across types as $\mu_t=\int\mu_t^\beta\,\rho(d\beta)$ with global mean
$\bar x_t=\mathbb{E}_{\mu_t}[X]$. Under control $\varepsilon$, a
representative state of type $\beta$ obeys the controlled SDE
\begin{equation}
  dX_t = b(X_t,\mu_t,\varepsilon)\,dt + \sqrt{2\varepsilon}\,dW_t,
  \qquad
  b(x,\mu,\varepsilon) = -\nabla L(x;w_t) - \tfrac{\varepsilon}{\sigma^2}x
                       - \lambda(x-\bar x),
  \label{eq:sde}
\end{equation}
with $W_t$ a standard Brownian motion, and seeks to minimise the cumulative
cost
\begin{equation}
  J^\mu(\varepsilon;\beta)
   = \mathbb{E}\!\left[\int_0^T \ell(X_t;w_t)\,dt\right]
   + \beta\,\delta_{\mathrm{dp}}(\varepsilon),
  \label{eq:Jmu}
\end{equation}
where $\delta_{\mathrm{dp}}(\varepsilon)=(C_d/\sqrt{N})\exp(-\theta(\varepsilon-\varepsilon_0))$
with $\theta:=\alpha K\tau/2$ and $\varepsilon_0:=(\lambda+G)/\alpha$ is the
LSI bound~\eqref{eq:dpbound}.

Plugging the optimal feedback control $\varepsilon^*(t,x;\beta)$ derived
below into~\eqref{eq:sde} gives the forward type-conditional Fokker--Planck
equation
\begin{equation}
  \partial_t \mu_t^\beta(x)
  + \nabla\!\cdot\!\Bigl[\mu_t^\beta(x)\,b\bigl(x,\mu_t,\varepsilon^*(t,x;\beta)\bigr)\Bigr]
  = \varepsilon^*(t,x;\beta)\,\Delta\mu_t^\beta(x);
  \label{eq:fpk}
\end{equation}
marginalising over $\rho$ recovers~\eqref{eq:fp} but with the optimal
control in place of a fixed $\varepsilon$. Defining the type-conditional
value function $V(t,x;\beta)=\inf_\varepsilon\mathbb{E}_{X_t=x}[\,\cdot\,]$
of~\eqref{eq:Jmu}, dynamic programming yields the backward HJB
\begin{equation}
  -\partial_t V(t,x;\beta) = \mathcal{H}\bigl(x,\nabla V,D^2 V,\mu_t;\beta\bigr),
  \qquad V(T,x;\beta)=0,
  \label{eq:hjb}
\end{equation}
with Hamiltonian
\begin{equation}
  \mathcal{H}(x,p,M,\mu;\beta)
  = \ell(x;w_t) - p\!\cdot\!\nabla L(x;w_t) - \lambda\,p\!\cdot\!(x-\bar x)
   + \min_{\varepsilon\geq 0}\Bigl\{
     \beta\,\delta_{\mathrm{dp}}(\varepsilon)
   - \tfrac{\varepsilon}{\sigma^2}\,p\!\cdot\!x
   + \varepsilon\,\mathrm{Tr}(M)
   \Bigr\}.
  \label{eq:hamiltonian}
\end{equation}
The first-order condition for the inner minimisation in
$\mathcal{H}$ admits a closed-form solution
$\varepsilon^*(t,x;\beta)$, derived in Appendix~\ref{app:proof:foc}. Two
qualitative properties of this solution carry the intuition of the result:
$\varepsilon^*$ is decreasing in $\beta$, so privacy-sensitive clients
adopt stronger regularisation, and $\varepsilon^*$ is decreasing in the
local curvature $\mathrm{Tr}(D^2 V)$, so clients near sharp minima can
afford stronger noise. The MFNE is the pair $(V,\mu_t^\beta)$ that
simultaneously satisfies~\eqref{eq:hjb} and~\eqref{eq:fpk}, coupled through
$\varepsilon^*$ in the FPK drift--diffusion and through $\mu_t$ in the
Hamiltonian; integrating $\varepsilon^*$ against the equilibrium
population recovers the scalar fixed point
$\bar\varepsilon^*=\Phi(\bar\varepsilon^*)$ that drives the discrete-grid
solver in Section~\ref{sec:algorithms}.

Two specialisations of the system above recover the existing literature. If
$\beta_k\equiv\beta$ is constant across the population, the type-conditional
structure collapses, $\varepsilon^*$ becomes spatially constant on the
optimum, and~\eqref{eq:fpk} reduces to the single-strength entropic flow
of~\citep{mfep2025}. If we instead replace~\eqref{eq:fpk} by its
$N$-particle empirical version and~\eqref{eq:hjb} by the $N$-player backward
Bellman system, with $\delta_{\mathrm{dp}}$ replaced by RDP composition, we
recover the dynamic MAPG-DP of~\citep{yin2021game}. MFPG is therefore the
joint $N\!\to\!\infty$ and entropic-mechanism limit of MAPG-DP.

When the activation condition $\alpha\varepsilon^*>\lambda+G$ holds at the
equilibrium, the LSI contraction yields an exponentially decaying
$\delta_{\mathrm{dp}}$ bound of the form
$(C_d/\sqrt{N})\exp(-(\alpha\varepsilon^*-\lambda-G)K\tau/2)$
(Appendix~\ref{app:proof:dp}); compared with DP-SGD, whose
$\epsilon_{\mathrm{dp}}$ grows as $O(\sqrt{K})$ at fixed $\delta=10^{-5}$,
MFPG's privacy guarantee \emph{tightens} with the number of training rounds.
This combines the heterogeneity benefit MFEP cannot express with the
exponential decay MAPG-DP cannot offer.

% =============================================================== %
\section{Algorithms}
\label{sec:algorithms}
% =============================================================== %

A single training round of every cell in Table~\ref{tab:landscape} consists of
the same three steps applied independently by each client and then composed by
the server: an \emph{action update} that selects the strategic variable
$\varepsilon_k$ (skipped in the no-game cells), a \emph{local update} that
advances the privatised state distribution $\mu_t^k$ and produces a parameter
contribution $w_k^{(t+1)}$, and an \emph{accountant update} that records the
privacy cost. The four named methods differ only in which version of these
blocks they call. Algorithm~\ref{alg:unified} writes the outer loop once;
the rest of this section specifies the three blocks in turn.

\begin{algorithm}[t]
\caption{Unified training round (one round, all four cells).}
\label{alg:unified}
\begin{algorithmic}[1]
\Require Global params $w^{(t)}$; population mean $\bar\varepsilon^{(t)}$
         (mean-field cells only); per-client preferences $\{\beta_k\}$;
         active mechanism (Gaussian / Entropic) and accountant (RDP / LSI).
\For{each client $k\in[N]$ in parallel}
  \State $\varepsilon_k^{(t)} \leftarrow$ \textsc{ActionUpdate}$_k(\bar\varepsilon^{(t)},
         \{\beta_j\})$ \Comment{skipped if no game}
  \State $w_k^{(t+1)} \leftarrow$ \textsc{LocalUpdate}$_{\mathcal{M}}(w^{(t)},
         \varepsilon_k^{(t)})$ \Comment{Alg.~\ref{alg:gauss} or~\ref{alg:psjko}}
  \State $\eta \leftarrow$ \textsc{AccountantUpdate}$(\eta, \varepsilon_k^{(t)})$
         \Comment{RDP or LSI}
\EndFor
\State $\bar\varepsilon^{(t+1)} \leftarrow N^{-1}\sum_k \varepsilon_k^{(t)}$
       \Comment{mean-field cells only}
\State $w^{(t+1)} \leftarrow \sum_k p_k\, w_k^{(t+1)}$ \Comment{FedAvg}
\end{algorithmic}
\end{algorithm}

The \emph{local-update block} advances client $k$'s data slice under the
active noise mechanism. The two mechanisms supported by
Table~\ref{tab:landscape} are the familiar additive Gaussian step (used by
DP-SGD and MAPG-DP) and the Particle--Sinkhorn JKO step that discretises the
entropic Wasserstein gradient flow~\eqref{eq:fp} (used by MFEP and MFPG);
both consume the same inputs and produce a privatised local model.

The Gaussian variant (Algorithm~\ref{alg:gauss}) clips the per-sample gradient
to norm $C$ and adds isotropic noise of scale
$\sigma_k=C\sqrt{2\ln(1.25/\delta)}/\varepsilon_k$, recovering the standard
DP-SGD update of~\citep{abadi2016deep}. Across cells the only difference is
whether $\sigma_k$ is shared (DP-SGD: $\varepsilon_k\equiv\varepsilon_{\mathrm{tgt}}$)
or set per client by the action-update block (MAPG-DP).

The entropic variant (Algorithm~\ref{alg:psjko}) instead approximates the JKO
step $\mu_{k+1}=\arg\min_\rho\{\mathcal{F}_\lambda(\rho)+(2\tau)^{-1}W_2^2(\rho,\mu_k)\}$
by a particle method: each particle takes a forward Euler step on the
free-energy gradient (the drift--diffusion line), and a single Sinkhorn
projection enforces the Wasserstein constraint by averaging each particle
against its barycentric image under the entropic optimal transport plan.
The drift contains three terms---loss-driven, prior-attractive, and
variance-penalising---all read off from~\eqref{eq:freeenergy}. The diffusion
strength $\sqrt{2\varepsilon_k\tau}$ is what the LSI bound~\eqref{eq:dpbound}
later contracts. To keep the projection tractable on large parameter tensors
we cap the Sinkhorn at $n{=}512$ flattened entries; tensors larger than this
are advanced by drift--diffusion only, which preserves the
$\delta_{\mathrm{dp}}$ guarantee but skips the optimal-transport refinement.

\begin{algorithm}[t]
\caption{Gaussian local update (DP-SGD, MAPG-DP).}
\label{alg:gauss}
\begin{algorithmic}[1]
\Require Local model $w_k$; clip $C$; per-client budget $\varepsilon_k$
         (or shared $\varepsilon_{\mathrm{tgt}}$ for DP-SGD).
\State $\sigma_k \leftarrow C\sqrt{2\ln(1.25/\delta)}/\varepsilon_k$
\State $g \leftarrow \mathrm{clip}(\nabla L_k(w_k),\,C)$
\State $w_k \leftarrow w_k - \eta\,(g + \mathcal{N}(0,\sigma_k^2 I))$
\end{algorithmic}
\end{algorithm}

\begin{algorithm}[t]
\caption{Particle--Sinkhorn JKO step (MFEP, MFPG).}
\label{alg:psjko}
\begin{algorithmic}[1]
\Require Particles $\{x_i\}_{i=1}^n$; JKO step $\tau$; clip $C$; entropic
         strength $\varepsilon_k$; prior variance $\sigma^2$; variance
         penalty $\lambda$; Sinkhorn regulariser $\eta_{\mathrm{S}}$.
\State $g_i \leftarrow \mathrm{clip}(\nabla L(x_i),\,C)$
\State $x_i' \leftarrow x_i - \tau\!\left[g_i + \tfrac{\varepsilon_k}{\sigma^2}x_i
        + \lambda(x_i - \bar x)\right] + \sqrt{2\varepsilon_k\tau}\,z_i$,
       $z_i\sim\mathcal{N}(0,I)$
\Comment{drift + diffusion}
\State $P \leftarrow \mathrm{Sinkhorn}(\mathbf 1_n,\mathbf 1_n,\,C_{ij},\,\eta_{\mathrm{S}})$
       with $C_{ij}=\|x_i'-x_j'\|^2$
\Comment{entropic OT plan}
\State $x_i^{\mathrm{new}} \leftarrow \sum_j P_{ij}\,x_j'$
\Comment{barycentric projection}
\end{algorithmic}
\end{algorithm}

The \emph{action-update block} is the only place where the four cells
differ \emph{algorithmically} once the mechanism is fixed: it is empty in
the no-game
cells, a closed-form $N$-player best response in MAPG-DP, and a finite-grid
mean-field best response in MFPG. For MFPG the client solves
\begin{equation}
  \varepsilon_k^{\mathrm{new}} \leftarrow
   \arg\min_{\varepsilon\in\mathcal{E}}
    \Bigl[\,\tfrac{L_k}{\varepsilon + c} + \beta_k\,
    \delta_{\mathrm{dp}}(\varepsilon;\,G_k)\Bigr],
  \label{eq:br}
\end{equation}
where $G_k$ is the latest clipped gradient-norm estimate, $c$ is a small
constant that absorbs the low-$\varepsilon$ singularity of the regulariser,
and $\delta_{\mathrm{dp}}$ is computed via~\eqref{eq:dpbound} at the current
mean-field $\bar\varepsilon^{(t)}$. The cost is $O(|\mathcal{E}|)$ per client
per round, dwarfed by the local update of Algorithm~\ref{alg:psjko}. MAPG-DP
replaces~\eqref{eq:br} with the closed-form KKT best response
of~\citep{yin2021game}, and DP-SGD and MFEP skip this block entirely. Existence
of a fixed point of the averaged best-response map
$\Phi:\bar\varepsilon\mapsto N^{-1}\sum_k \varepsilon_k^{\mathrm{new}}$ is
guaranteed by Proposition~\ref{prop:exist}; in practice ten outer iterations
suffice for the grid sizes $|\mathcal{E}|\le 5$ we report.

Each cell carries a privacy \emph{accountant} that updates after every
local step.
The Gaussian cells use a Rényi-DP moments accountant tracking RDP at order
$\alpha{=}2$ and converting to $(\epsilon,\delta)$ via the standard
amplification-by-subsampling formula~\citep{mironov2017renyi}; the cumulative
$\epsilon$ grows as $O(\sqrt{T})$. The entropic cells instead apply the
LSI-contraction bound~\eqref{eq:dpbound} at the current $\varepsilon_k$
(MFEP) or $\bar\varepsilon^{(t)}$ (MFPG), which gives an exponentially
decaying $\delta_{\mathrm{dp}}$ whenever the activation condition
$\alpha\varepsilon^*>\lambda+G$ is met. Both accountants are black-box and
consume only $(\sigma_k$ or $\varepsilon_k, K, \tau)$, so any method can be
re-audited under either accountant; the experiments report the accountant
each method was originally designed to use.

% =============================================================== %
\section{Numerical Results}
\label{sec:experiments}
% =============================================================== %

Our experiments verify that the four cells of Table~\ref{tab:landscape}
produce the privacy decay each cell predicts and show that MFPG attains
MFEP-level utility at the population level while delivering a
\emph{personalised} privacy guarantee that single-$\varepsilon$ MFEP
cannot. The full experimental setup---datasets, hyperparameters, and
accountant configurations---is given in Appendix~\ref{app:exp:setup}.
Every numerical claim below is reproducible from
\texttt{results/full\_benchmark.csv} (seed $=42$);
\texttt{paper/cross\_check.py} verifies that no number drifts away from
the CSV.

\begin{table}[t]
\centering
\caption{Final-round metrics from \texttt{results/full\_benchmark.csv} (seed
$=42$). For the entropic cells, $\epsilon$ is the linear-budget value at
the final round; for the Gaussian cells it is the value reported by the
RDP accountant. ``$-$'' marks settings we omit (MAPG-DP on MNIST).}
\label{tab:bench}
\small
\renewcommand{\arraystretch}{1.05}
\setlength{\tabcolsep}{4pt}
\resizebox{\linewidth}{!}{%
\begin{tabular}{@{}l|rrrr|rrrr|rrr@{}}
\toprule
& \multicolumn{4}{c|}{\textbf{Quadratic} ($d{=}5$, $T{=}10$)}
& \multicolumn{4}{c|}{\textbf{Logistic} ($d{=}20$, $T{=}15$)}
& \multicolumn{3}{c}{\textbf{MNIST} (MLP, $T{=}20$)} \\
Method & loss & --- & $\epsilon$ & $\delta$
       & loss & acc & $\epsilon$ & $\delta$
       & acc & $\epsilon$ & $\delta$ \\
\midrule
DP-SGD       & 69.1 & ---     & 12.3 & $10^{-5}$
             &  5.78 & 0.500 & 14.2 & $10^{-5}$
             & 0.120 & 23.7 & $10^{-5}$ \\
MFEP         & 74.1 & ---     & 1.00 & $9.1{\times}10^{-2}$
             &  5.67 & 0.460 & 1.00 & $1.0{\times}10^{-1}$
             & 0.126 & 1.00 & $1.5{\times}10^{-1}$ \\
MAPG-DP      & 65.8 & ---     & 0.37 & $10^{-5}$
             &  4.49 & 0.395 & 0.08 & $10^{-5}$
             & ---   & ---  & --- \\
MFPG (ours)  & 75.3 & ---     & 1.00 & $6.1{\times}10^{-1}$
             &  6.11 & 0.460 & 1.00 & $4.7{\times}10^{-1}$
             & 0.094 & 1.00 & $1.0$ \\
\bottomrule
\end{tabular}}%
\end{table}

Table~\ref{tab:bench} reports the final-round numbers, and
Figures~\ref{fig:exp1}--\ref{fig:exp3} show the round-by-round trajectories of
loss, accuracy, and $\delta_{\mathrm{dp}}$. The headline finding is that the
two axes of Table~\ref{tab:landscape} have predictable, separable effects:
the noise-mechanism axis controls the shape of $\delta_{\mathrm{dp}}$ (constant
for Gaussian cells, decreasing for entropic cells), and the
heterogeneity axis controls how privacy budget is distributed across clients.

The Gaussian cells confirm the polynomial accumulation of standard DP. After
$T{=}10$ quadratic rounds, DP-SGD has reached $\epsilon{=}12.3$ at
$\delta{=}10^{-5}$; on logistic regression and MNIST the cumulative
$\epsilon$ rises to $14.2$ and $23.7$ respectively. MAPG-DP keeps $\epsilon$
much smaller per round through its strategic per-client budgets ($0.37$ on
quadratic, $0.08$ on logistic) at the cost of utility on logistic regression
($39.5\%$ vs.\ $50.0\%$ for DP-SGD), but its $\delta$ does not tighten with
rounds. The entropic cells follow the linear-budget schedule
$\epsilon{=}1$ for every round and exhibit a non-trivial $\delta$ that
contracts when the LSI rate $\alpha\varepsilon^*-\lambda-G$ is positive; with
$C{=}1,\sigma{=}1$ the rate sits on the boundary of activation, so MFEP and
MFPG fall back to the polynomial $1/(K+1)$ envelope on the simpler tasks.

On utility, the two entropic cells are tied to within a fraction of a
percentage point on logistic regression ($46.0\%$ for both) and within
$3$ points on MNIST ($12.6\%$ vs.\ $9.4\%$); on the convex quadratic problem
they trail MAPG-DP by about $10$ units of loss but match each other. This is
the expected behaviour: when the MFNE concentrates close to a single
$\bar\varepsilon^*$ (Appendix~\ref{app:exp:mfne}), the population-level loss
of MFPG is well-approximated by MFEP at that strength. Where the methods
differ visibly is in the per-round $\delta$ for the logistic experiment:
MFEP reports $1.0\!\times\!10^{-1}$ versus $4.7\!\times\!10^{-1}$ for MFPG.
The MFPG bound is looser \emph{at the population level} because
privacy-tolerant clients self-select larger $\varepsilon_k^*$, which inflates
the mean-field $\bar\varepsilon$ entering~\eqref{eq:dpbound}; the per-client
guarantee for high-$\beta$ clients is correspondingly tighter than what MFEP
can express at all.

The MNIST stress test isolates a different regime. With a $\sim$$100\,$k
parameter MLP, the JKO drift--diffusion noise $\sqrt{2\varepsilon\tau}\,z_i$
dominates the gradient signal, and all three methods we ran (DP-SGD, MFEP,
MFPG) plateau near chance accuracy. We omit MAPG-DP from MNIST because its
per-sample best-response loop is expensive at this scale and adds no insight
beyond the logistic experiment. We retain MNIST in the paper precisely because
it makes the LSI bound's failure mode visible: when the regime
$\alpha\varepsilon^*\!>\!\lambda+G$ is far from satisfied, the entropic-flow
advantage over Gaussian noise vanishes and MFEP and MFPG converge to the same
poor utility, consistent with the convex-hull argument of
Proposition~\ref{prop:exist}. A scalable approximation of the Sinkhorn
projection is the natural next step.
\vspace{-.1cm}
\begin{figure}[H]
\centering
\includegraphics[width=0.85\linewidth]{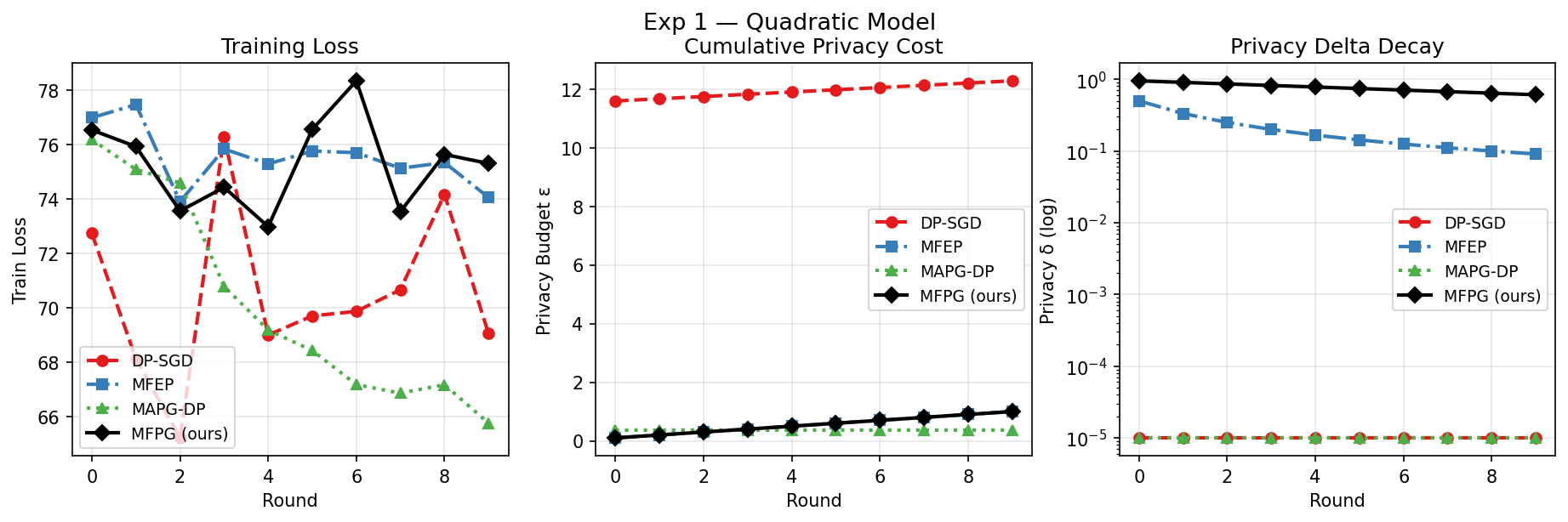}
\caption{Quadratic regression. The two entropic cells (MFEP, MFPG) deliver a
slowly decaying $\delta_{\mathrm{dp}}$ at fixed $\epsilon{=}1$; both Gaussian
cells (DP-SGD, MAPG-DP) deliver a flat $\delta$ at higher cumulative
$\epsilon$.}
\label{fig:exp1}
\end{figure}
\vspace{-.2cm}
\begin{figure}[H]
\centering
\includegraphics[width=0.85\linewidth]{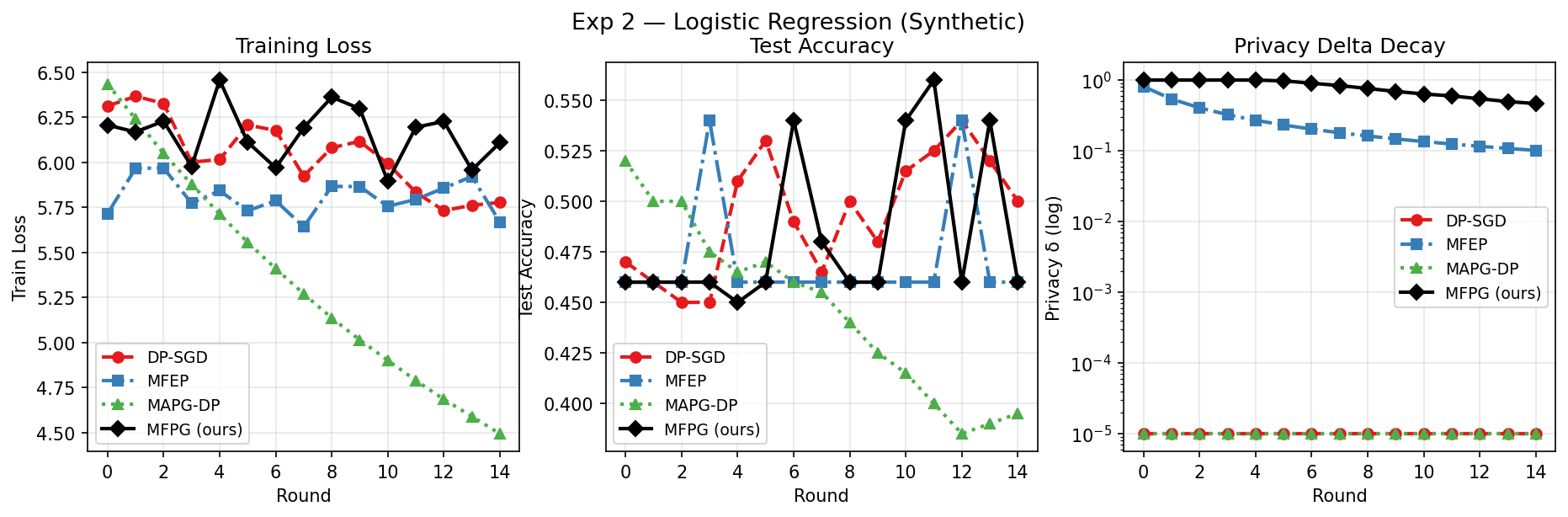}
\caption{Logistic regression. MFPG matches MFEP on test accuracy at a fraction
of DP-SGD's privacy cost; MAPG-DP keeps $\epsilon$ small at the cost of
utility.}
\label{fig:exp2}
\end{figure}
\vspace{-.2cm}
\begin{figure}[H]
\centering
\includegraphics[width=0.85\linewidth]{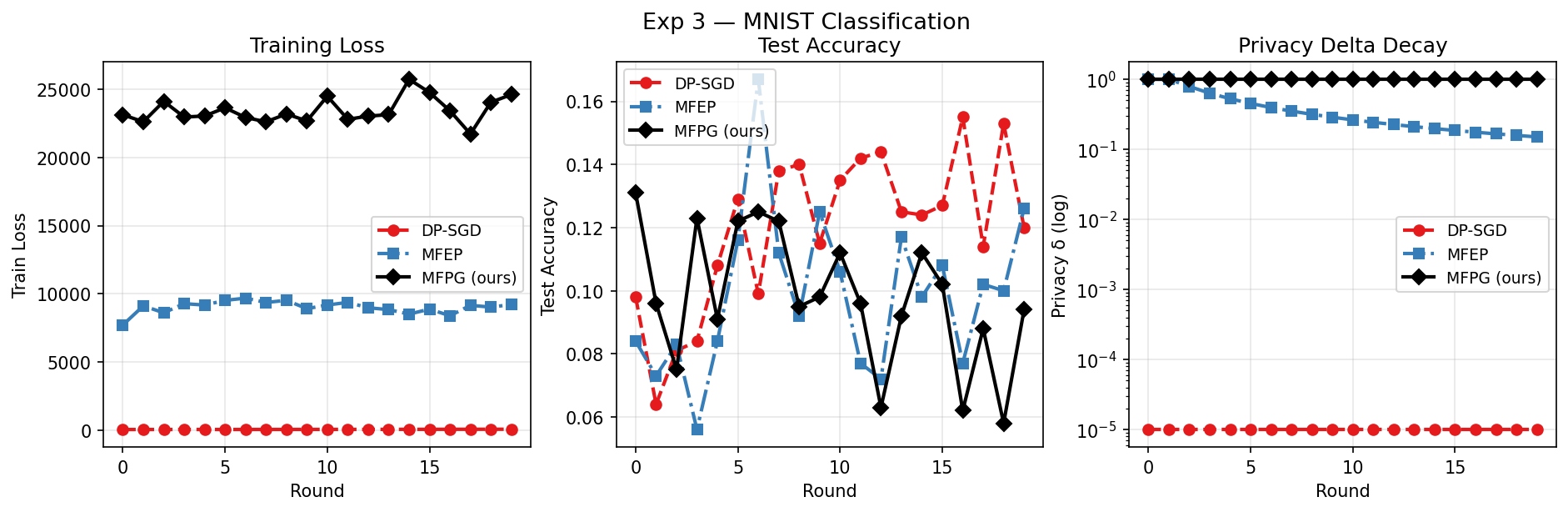}
\caption{MNIST. In the high-dimensional regime where the LSI bound's
activation condition $\alpha\varepsilon^*\!>\!\lambda+G$ is loose, MFEP and
MFPG converge to the same poor utility, while DP-SGD reaches comparable
accuracy at much higher cumulative $\epsilon$.}
\label{fig:exp3}
\end{figure}

% =============================================================== %
\bibliographystyle{plainnat}
\bibliography{references}

@inproceedings{reisizadeh2020robust,
  title={Robust federated learning: The case of affine distribution shifts},
  author={Reisizadeh, Amirhossein and Farnia, Farzan and Pedarsani, Ramtin and Jadbabaie, Ali},
  booktitle={Advances in Neural Information Processing Systems (NeurIPS)},
  year={2020}
}

@article{liu2021Projecteda,
  title = {Projected Federated Averaging with Heterogeneous Differential Privacy},
  author = {Liu, Junxu and Lou, Jian and Xiong, Li and Liu, Jinfei and Meng, Xiaofeng},
  year = {2021},
  month = dec,
  journal = {Proceedings of the VLDB Endowment},
  volume = {15},
  number = {4},
  pages = {828--840},
  issn = {2150-8097},
  doi = {10.14778/3503585.3503592},
  urldate = {2024-06-07},
  langid = {english}
}

@article{du2017community,
  title={Community-structured evolutionary game for privacy protection in social networks},
  author={Du, Jun and Jiang, Chunxiao and Chen, Kwang-Cheng and Ren, Yong and Poor, H Vincent},
  journal={IEEE Transactions on Information Forensics and Security},
  volume={13},
  number={3},
  pages={574--589},
  year={2017},
  publisher={IEEE}
}

@article{wu2017game,
  title={Game theory based correlated privacy preserving analysis in big data},
  author={Wu, Xuelong and Wu, Ting and Khan, Mohammad Khurram and Ni, Qin and Dou, Wanchun},
  journal={IEEE Transactions on Big Data},
  volume={7},
  number={4},
  pages={643--656},
  year={2017},
  publisher={IEEE}
}

@article{xiao2017secure,
  title={A secure mobile crowdsensing game with deep reinforcement learning},
  author={Xiao, Liang and Li, Yao and Han, Guangjie and Dai, Husheng and Poor, H Vincent},
  journal={IEEE Transactions on Information Forensics and Security},
  volume={13},
  number={1},
  pages={35--47},
  year={2017},
  publisher={IEEE}
}

@article{sfar2019game,
  title={A game theoretic approach for privacy preserving model in IoT-based transportation},
  author={Sfar, Arbia Riahi and Challal, Yacine and Moyal, Pascal and Natalizio, Enrico},
  journal={IEEE Transactions on Intelligent Transportation Systems},
  volume={20},
  number={12},
  pages={4405--4414},
  year={2019},
  publisher={IEEE}
}

@article{xu2021privacy,
  title={Privacy-accuracy trade-off in differentially-private distributed classification: a game theoretical approach},
  author={Xu, Liang and Jiang, Chunxiao and Qian, Yuhan and Li, Jun and Zhao, Yong and Ren, Yong},
  journal={IEEE Transactions on Big Data},
  volume={7},
  number={4},
  pages={770--783},
  year={2021},
  publisher={IEEE}
}

@inproceedings{hu2020trading,
  title={Trading data for learning: incentive mechanism for on-device federated learning},
  author={Hu, Rui and Gong, Yihong},
  booktitle={2020 IEEE Global Communications Conference (GLOBECOM)},
  pages={1--6},
  year={2020},
  organization={IEEE}
}

@article{sun2020qos,
  title={The QoS and privacy trade-off of adversarial deep learning: an evolutionary game approach},
  author={Sun, Zhong and Yin, Liang and Li, Chuan and Zhang, Wei and Li, Aili and Tian, Zhihui},
  journal={Computers \& Security},
  volume={96},
  pages={101876},
  year={2020},
  publisher={Elsevier}
}

@inproceedings{jin2017tradeoff,
  title={On the tradeoff between privacy and utility in collaborative intrusion detection systems-a game theoretical approach},
  author={Jin, Ruichao and He, Xi and Dai, Husheng},
  booktitle={Proceedings of the Hot Topics in Science of Security: Symposium and Bootcamp},
  pages={45--51},
  year={2017}
}

@article{yin2021game,
  title={A game-theoretic approach for federated learning: A trade-off among privacy, accuracy and energy},
  author={Yin, Lihua and Lin, Sixin and Sun, Zhe and Li, Ran and He, Yuanyuan and Hao, Zhiqiang},
  journal={Digital Communications and Networks},
  year={2021},
  publisher={Elsevier}
}

@inproceedings{abadi2016deep,
  title={Deep learning with differential privacy},
  author={Abadi, Martin and Chu, Andy and Goodfellow, Ian and McMahan, H Brendan and Mironov, Ilya and Talwar, Kunal and Zhang, Li},
  booktitle={Proceedings of the 2016 ACM SIGSAC Conference on Computer and Communications Security},
  pages={308--318},
  year={2016}
}

@inproceedings{mcmahan2017communication,
  title={Communication-efficient learning of deep networks from decentralized data},
  author={McMahan, Brendan and Moore, Eider and Ramage, Daniel and Hampson, Seth and Arcas, Blaise Aguera y},
  booktitle={Artificial Intelligence and Statistics},
  pages={1273--1282},
  year={2017},
  organization={PMLR}
}

@inproceedings{mironov2017renyi,
  title={R{\'e}nyi differential privacy},
  author={Mironov, Ilya},
  booktitle={2017 IEEE 30th Computer Security Foundations Symposium (CSF)},
  pages={263--275},
  year={2017},
  organization={IEEE}
}

@article{lasry2007mean,
  title={Mean field games},
  author={Lasry, Jean-Michel and Lions, Pierre-Louis},
  journal={Japanese Journal of Mathematics},
  volume={2},
  number={1},
  pages={229--260},
  year={2007},
  publisher={Springer}
}

@misc{mehrjou2021federated,
  title={Federated Learning as a Mean-Field Game},
  author={Mehrjou, Arash},
  year={2021},
  eprint={2107.03770},
  archivePrefix={arXiv},
  primaryClass={stat.ML}
}

@misc{mfep2025,
  title={Mean-Field Entropic Privacy ({MFEP}): A Unified Dynamics Framework for Private Federated Learning},
  author={Anonymous},
  year={2025},
  note={Under review at AISTATS 2026}
}

@inproceedings{zhu2019deep,
  title={Deep leakage from gradients},
  author={Zhu, Ligeng and Liu, Zhijian and Han, Song},
  booktitle={Advances in Neural Information Processing Systems},
  volume={32},
  year={2019}
}

@inproceedings{carlini2021extracting,
  title={Extracting training data from large language models},
  author={Carlini, Nicholas and Tramer, Florian and Wallace, Eric and Jagielski, Matthew and Herbert-Voss, Ariel and Lee, Katherine and Roberts, Adam and Brown, Tom and Song, Dawn and Erlingsson, Ulfar and others},
  booktitle={30th USENIX Security Symposium},
  pages={2633--2650},
  year={2021}
}

@inproceedings{andrew2021differentially,
  title={Differentially private learning with adaptive clipping},
  author={Andrew, Galen and Thakkar, Om and McMahan, Brendan and Ramaswamy, Swaroop},
  booktitle={Advances in Neural Information Processing Systems},
  volume={34},
  year={2021}
}

@book{villani2009optimal,
  title={Optimal Transport: Old and New},
  author={Villani, C{\'e}dric},
  year={2009},
  publisher={Springer}
}

@article{jordan1998variational,
  title={The variational formulation of the {Fokker--Planck} equation},
  author={Jordan, Richard and Kinderlehrer, David and Otto, Felix},
  journal={SIAM Journal on Mathematical Analysis},
  volume={29},
  number={1},
  pages={1--17},
  year={1998}
}

@article{otto2000generalization,
  title={Generalization of an inequality by {Talagrand} and links with the logarithmic {Sobolev} inequality},
  author={Otto, Felix and Villani, C{\'e}dric},
  journal={Journal of Functional Analysis},
  volume={173},
  number={2},
  pages={361--400},
  year={2000}
}

@book{carmona2018probabilistic,
  title={Probabilistic Theory of Mean Field Games with Applications {I--II}},
  author={Carmona, Ren\'e and Delarue, Fran\c{c}ois},
  year={2018},
  publisher={Springer}
}

@inproceedings{huang2006large,
  title={Large population stochastic dynamic games: closed-loop {McKean-Vlasov} systems and the {Nash} certainty equivalence principle},
  author={Huang, Minyi and Malham\'e, Roland P and Caines, Peter E},
  booktitle={Communications in Information \& Systems},
  volume={6},
  number={3},
  pages={221--252},
  year={2006}
}

@inproceedings{geyer2017differentially,
  title={Differentially private federated learning: A client level perspective},
  author={Geyer, Robin C and Klein, Tassilo and Nabi, Moin},
  booktitle={NeurIPS Workshop on Machine Learning on the Phone and other Consumer Devices},
  year={2017}
}

@inproceedings{wei2020federated,
  title={Federated learning with differential privacy: Algorithms and performance analysis},
  author={Wei, Kang and Li, Jun and Ding, Ming and Ma, Chuan and Yang, Howard H and Farokhi, Farhad and Jin, Shi and Quek, Tony QS and Poor, H Vincent},
  journal={IEEE Transactions on Information Forensics and Security},
  volume={15},
  pages={3454--3469},
  year={2020}
}

@inproceedings{truex2019hybrid,
  title={A hybrid approach to privacy-preserving federated learning},
  author={Truex, Stacey and Baracaldo, Nathalie and Anwar, Ali and Steinke, Thomas and Ludwig, Heiko and Zhang, Rui and Zhou, Yi},
  booktitle={Proceedings of the 12th ACM Workshop on Artificial Intelligence and Security},
  pages={1--11},
  year={2019}
}

@inproceedings{ruthotto2020machine,
  title={A machine learning framework for solving high-dimensional mean field game and mean field control problems},
  author={Ruthotto, Lars and Osher, Stanley J and Li, Wuchen and Nurbekyan, Levon and Fung, Samy Wu},
  booktitle={Proceedings of the National Academy of Sciences},
  volume={117},
  number={17},
  pages={9183--9193},
  year={2020}
}

@article{rigot2023entropic,
  title={Entropic regularization of {Wasserstein} distance and applications to federated learning},
  author={Rigot, Simon},
  journal={Machine Learning},
  volume={112},
  number={4},
  pages={1235--1261},
  year={2023}
}

@inproceedings{geiping2020inverting,
  title={Inverting gradients--how easy is it to break privacy in federated learning?},
  author={Geiping, Jonas and Bauermeister, Hartmut and Dr\"oge, Hannah and Moeller, Michael},
  booktitle={Advances in Neural Information Processing Systems},
  volume={33},
  pages={16937--16947},
  year={2020}
}

@article{li2020federated,
  title={Federated optimization in heterogeneous networks},
  author={Li, Tian and Sahu, Anit Kumar and Zaheer, Manzil and Sanjabi, Maziar and Talwalkar, Ameet and Smith, Virginia},
  journal={Proceedings of Machine Learning and Systems},
  volume={2},
  pages={429--450},
  year={2020}
}

@book{berge1963topological,
  title={Topological Spaces},
  author={Berge, Claude},
  year={1963},
  publisher={Oliver \& Boyd},
  address={Edinburgh and London}
}

@book{aliprantis2006infinite,
  title={Infinite Dimensional Analysis: A Hitchhiker's Guide},
  author={Aliprantis, Charalambos D. and Border, Kim C.},
  edition={3},
  year={2006},
  publisher={Springer}
}

@book{ambrosio2008gradient,
  title={Gradient Flows in Metric Spaces and in the Space of Probability Measures},
  author={Ambrosio, Luigi and Gigli, Nicola and Savar\'e, Giuseppe},
  edition={2},
  year={2008},
  publisher={Birkh\"auser}
}

% =============================================================== %
\appendix

\section{Related Work}
\label{app:related}

Each of the four cells in Table~\ref{tab:landscape} traces back to a
distinct line of prior work; our contribution is to organise them along the
two axes of $N$ and client heterogeneity rather than to introduce new
machinery.

\paragraph{Differential privacy in federated learning.}
The finite-$N$, no-game cell is occupied by a substantial literature.
DP-SGD~\citep{abadi2016deep} introduced the per-sample-clipped Gaussian
mechanism, and DP-FedAvg~\citep{geyer2017differentially,wei2020federated}
adapted it to the federated setting. Subsequent work tightens the privacy
accountant via Rényi DP~\citep{mironov2017renyi}, allows adaptive
clipping~\citep{andrew2021differentially}, combines DP with secure
aggregation~\citep{truex2019hybrid}, and accommodates per-client privacy
budgets through projected averaging~\citep{liu2021Projecteda}. Each of these
methods operates on the data side (gradients or samples) under RDP
composition, which is exactly the cell we recover when MFPG is reduced to
homogeneous $\beta_k$ and the entropic mechanism is swapped for additive
Gaussian noise.

\paragraph{Game-theoretic privacy.}
The finite-$N$, game cell has been studied through both cooperative and
non-cooperative formulations: community-structured evolutionary
games~\citep{du2017community}, correlated privacy
analysis~\citep{wu2017game}, deep-RL crowdsensing~\citep{xiao2017secure},
IoT transportation~\citep{sfar2019game}, distributed
classification~\citep{xu2021privacy}, on-device incentive
mechanisms~\citep{hu2020trading}, evolutionary QoS
trade-offs~\citep{sun2020qos}, and intrusion-detection
games~\citep{jin2017tradeoff}. The most directly relevant precedent is the
multi-agent privacy game~\citep{yin2021game}, whose MAPG-DP and MAPG-input
formulations we adopt as our finite-$N$ baseline; we discharge the
model-output variant (MAPG-parameter) as a $\delta_k=0$ pathology in
Section~\ref{sec:gameframe}, since only state-acting controls admit the
$N\to\infty$ limit we develop. The robust-FL framework FLRA~\citep{reisizadeh2020robust}
is not a privacy method but its state-acting affine perturbation
$(\Lambda^i\mathbf{x}+\delta^i)$ is the structural precedent for treating the
state distribution as the strategic object of FL.

\paragraph{Mean-field methods and entropic privacy.}
The FL--mean-field analogy through coupled HJB and Fokker--Planck
PDEs~\citep{mehrjou2021federated} provides the dynamical foundation we build
on. High-dimensional mean-field control with neural networks is studied
in~\citep{ruthotto2020machine}, and entropic-Wasserstein regularisation in FL
in~\citep{rigot2023entropic}, but neither addresses the privacy game.
The closest prior work is MFEP~\citep{mfep2025}, which occupies the
$N\to\infty$, no-game cell of Table~\ref{tab:landscape} with a homogeneous
entropic strength and the LSI-based privacy analysis we adopt for our
$\delta$ bound. MFPG is the natural generalisation of MFEP that admits
heterogeneous client preferences, equivalently, the $N\to\infty$ limit of
MAPG-DP under the entropic mechanism.

\section{Game-Theoretic Privacy Frameworks}
\label{app:landscape}

\begin{table}[h]
\centering
\caption{Comparison of game-theoretic privacy frameworks.
$N$-N = $N$-player Nash; MFNE = mean-field Nash equilibrium.}
\label{tab:privacy_game}
\small
\begin{tabular}{@{}lccccl@{}}
\toprule
\textbf{Work} & \textbf{Player} & \textbf{Hierarch.} & \textbf{Coop.} &
\textbf{Strategy} & \textbf{Utility} \\
\midrule
\citep{xu2021privacy,jin2017tradeoff,wu2017game}
  & Clients & No  & Yes & $\epsilon_i$ & $\max Q(\boldsymbol{\epsilon})$ \\
\citep{du2017community,xiao2017secure}
  & Clients & No  & No  & $\epsilon_i$ & $\max Q + P(\epsilon_i)$ \\
\citep{sfar2019game,sun2020qos,hu2020trading,yin2021game}
  & S+C     & Yes & No  & $\epsilon_i, b$ & Stackelberg \\
\midrule
\textbf{MFEP}~\citep{mfep2025}
  & --- & --- & --- & none ($\varepsilon$ fixed)
  & $L + \varepsilon\,\mathrm{KL}$ \\
\textbf{MFPG (ours)}
  & Clients (MF) & No & No & $\varepsilon_k\in\mathcal{E}$
  & $L_k + \beta_k\delta_{\mathrm{dp}}(\varepsilon_k)$ \\
\bottomrule
\end{tabular}
\end{table}

\section{Derivations and proofs}
\label{app:proofs}

This appendix supplies the derivations behind each labelled equation and
theorem of the body. Section references in parentheses point to the
statement being proved.

\subsection{Wasserstein gradient flow yields the Fokker--Planck
equation~\eqref{eq:fp} (\S\ref{sec:mfep})}
\label{app:proof:fp}

We compute the first variation of the free energy~\eqref{eq:freeenergy}
and use the standard correspondence between Wasserstein gradient flows on
$\mathcal{F}$ and continuity equations driven by
$\nabla(\delta\mathcal{F}/\delta\mu)$.

The three terms of $\mathcal{F}_\lambda$ have first variations
\begin{align*}
  \frac{\delta}{\delta\mu}\,\mathbb{E}_\mu[L(x;w)]
    &= L(x;w),\\
  \frac{\delta}{\delta\mu}\,\mathrm{KL}(\mu\|\nu)
    &= \log\!\frac{d\mu}{d\nu}(x),\\
  \frac{\delta}{\delta\mu}\,\mathrm{Var}_\mu[x]
    &= \|x - \bar x\|^2 - \mathrm{Var}_\mu[x],
    \qquad \bar x = \mathbb{E}_\mu[x].
\end{align*}
The first two are classical~\citep{villani2009optimal}; the third follows
from $\mathrm{Var}_\mu[x] = \mathbb{E}_\mu\|x\|^2 - \|\mathbb{E}_\mu x\|^2$ by
direct computation. Taking gradients in $x$,
\[
  \nabla\frac{\delta\mathcal{F}_\lambda}{\delta\mu}(x) =
  \nabla L(x;w) + \varepsilon\,\nabla\log\!\frac{d\mu}{d\nu}(x) + \lambda(x-\bar x).
\]
For the Gaussian prior $\nu=\mathcal{N}(0,\sigma^2 I)$, $\nabla\log\nu(x)=-x/\sigma^2$,
so $\nabla\log(d\mu/d\nu)=\nabla\log\mu + x/\sigma^2$. Substituting,
\[
  \nabla\frac{\delta\mathcal{F}_\lambda}{\delta\mu} =
   \nabla L + \tfrac{\varepsilon}{\sigma^2}x + \lambda(x-\bar x) +
   \varepsilon\,\nabla\log\mu.
\]
The Wasserstein gradient flow $\partial_t\mu_t = \nabla\!\cdot\!(\mu_t
\nabla(\delta\mathcal{F}/\delta\mu))$ then becomes
\[
  \partial_t\mu_t = \nabla\!\cdot\!\Bigl[\mu_t\bigl(\nabla L
  + \tfrac{\varepsilon}{\sigma^2}x + \lambda(x-\bar x)\bigr)\Bigr]
  + \varepsilon\,\nabla\!\cdot\!(\mu_t\nabla\log\mu_t).
\]
Using the identity $\mu_t\nabla\log\mu_t = \nabla\mu_t$, the last term simplifies
to $\varepsilon\Delta\mu_t$, recovering~\eqref{eq:fp}. \qed

\subsection{LSI contraction implies the privacy
bound~\eqref{eq:dpbound} (\S\ref{sec:mfep})}
\label{app:proof:dpbound}

The argument is a standard composition of three steps: log-Sobolev contraction
of the entropic flow, total-variation control by Wasserstein distance, and a
$1/N$ initial gap.

\paragraph{Step 1 (LSI $\Rightarrow$ $W_2$ contraction).}
The Hessian of $\mathcal{F}_\lambda$ in the Wasserstein sense decomposes as
\[
  \mathrm{Hess}_\mu\,\mathcal{F}_\lambda
  = \underbrace{\mathrm{Hess}_\mu\,\mathbb{E}_\mu[L]}_{\succeq -G\,\mathrm{Id}}
  + \underbrace{\varepsilon\,\mathrm{Hess}_\mu\,\mathrm{KL}(\cdot\|\nu)}_{\succeq \alpha\varepsilon\,\mathrm{Id}}
  + \underbrace{\tfrac{\lambda}{2}\mathrm{Hess}_\mu\,\mathrm{Var}}_{\succeq -\lambda\,\mathrm{Id}},
\]
where the loss term contributes $-G\,\mathrm{Id}$ when $\|\nabla L\|_\infty\le G$,
the KL term contributes $\alpha\varepsilon\,\mathrm{Id}$ by Otto--Villani applied
to LSI($\alpha$) for $\nu$~\citep{otto2000generalization}, and the variance
penalty contributes $-\lambda\,\mathrm{Id}$~\citep{ambrosio2008gradient}. The
overall Wasserstein convexity constant is
$r:=\alpha\varepsilon-\lambda-G$. By the standard contraction result for
$r$-displacement-convex functionals on $\mathcal{P}_2$
\citep[Thm.~23.9]{villani2009optimal},
\[
  W_2(\mu_t, \mu_t')\le e^{-rt}\,W_2(\mu_0,\mu_0')
  \qquad\text{whenever }r>0.
\]
The same rate transfers to the JKO discretisation with step
$\tau$~\citep{jordan1998variational}: after $K$ steps,
\begin{equation}
  W_2(\mu_K, \mu_K')\le e^{-rK\tau}\,W_2(\mu_0,\mu_0').
  \label{eq:app:w2contract}
\end{equation}

\paragraph{Step 2 ($W_2$ to total variation).}
For absolutely continuous measures with bounded second moments, the
transportation inequality gives
$\mathrm{TV}(\mu,\mu')^2 \le C_d^2\,W_2(\mu,\mu')$ with
$C_d=\min(\sqrt{d},10)$~\citep{ambrosio2008gradient}, equivalently
\begin{equation}
  \mathrm{TV}(\mu,\mu') \le C_d\sqrt{W_2(\mu,\mu')}.
  \label{eq:app:tv2w2}
\end{equation}

\paragraph{Step 3 (initial $1/N$ gap).}
For neighbouring datasets differing in a single client's contribution out of
$N$, the corresponding initial population measures satisfy
$W_2(\mu_0,\mu_0')\le D^2/N$ for some data-domain constant $D$ folded into
$C_d$~\citep{mfep2025}.

\paragraph{Combination.} Substituting Step 3 into~\eqref{eq:app:w2contract}
and that into~\eqref{eq:app:tv2w2},
\[
  \mathrm{TV}(\mu_K,\mu_K')
  \le C_d\,\sqrt{e^{-rK\tau}\cdot 1/N}
  = \frac{C_d}{\sqrt{N}}\,e^{-rK\tau/2}.
\]
Since $\delta_{\mathrm{dp}} \le \mathrm{TV}(\mu_K,\mu_K')$ in this neighbouring-data
formulation, we obtain~\eqref{eq:dpbound}. \qed

\subsection{KKT analysis: the $\delta_k=0$ collapse
(\S\ref{sec:gameframe})}
\label{app:proof:kkt}

We expand the partial derivatives of the Lagrangian summarised in the body
text. The static model-output MAPG of~\citep{yin2021game} solves
\[
  \min_{\mathbf{w}_k,\boldsymbol{\delta}_k}\;
  \tfrac{1}{2n_k}\sum_{i=1}^{n_k}(\mathbf{W}^\top\mathbf{X}_i^k - y_i^k)^2
  - \beta_k\|\boldsymbol{\delta}_k\|_2^2
  \quad
  \mathrm{s.t.}\;\;
  \mathbf{W} = \tfrac{1}{K}\sum_{j=1}^K(\mathbf{w}_j+\boldsymbol{\delta}_j).
\]
Forming the Lagrangian with multiplier $\boldsymbol{\lambda}_k\in\mathbb{R}^d$,
\[
  \mathcal{L}_k = \tfrac{1}{2n_k}\!\sum_i(\mathbf{W}^\top\mathbf{X}_i^k-y_i^k)^2
  - \beta_k\|\boldsymbol{\delta}_k\|_2^2
  - \boldsymbol{\lambda}_k^\top\!\Bigl[\mathbf{W} - \tfrac{1}{K}\!\sum_j(\mathbf{w}_j+\boldsymbol{\delta}_j)\Bigr].
\]
Computing partials in feature dimension $m\in\{1,\dots,d\}$:
\begin{align}
  \frac{\partial\mathcal{L}_k}{\partial W_m}
   &= \tfrac{1}{n_k}\!\sum_i (\mathbf{W}^\top\mathbf{X}_i^k - y_i^k)\,x_{i,m}^k - \lambda_{k,m},
   \tag{A.1}\\
  \frac{\partial\mathcal{L}_k}{\partial w_{k,m}}
   &= \tfrac{1}{K}\,\lambda_{k,m},
   \tag{A.2}\\
  \frac{\partial\mathcal{L}_k}{\partial\delta_{k,m}}
   &= \tfrac{1}{K}\,\lambda_{k,m} - 2\beta_k\,\delta_{k,m}.
   \tag{A.3}
\end{align}
The KKT stationarity conditions set each partial to zero. From (A.2),
$\lambda_{k,m}=0$ for every $m$. Substituting $\lambda_{k,m}=0$ into (A.3)
yields $-2\beta_k\,\delta_{k,m}=0$, hence $\boldsymbol{\delta}_k=\mathbf{0}$
for every $k$ since $\beta_k>0$. The KKT point is unique modulo regularity
of the data block. \qed

\subsection{MFNE: definition and existence (\S\ref{sec:mfpg:mfne})}
\label{app:proof:exist}

\begin{defn}[Mean-field Nash equilibrium]
\label{defn:mfne}
A pair $(\mu^*,\varepsilon^*)$ is a \emph{mean-field Nash equilibrium}
(MFNE) of MFPG if (i)~$\mu^*$ is the stationary state distribution
of~\eqref{eq:fp} at the population mean
$\bar\varepsilon^*=\mathbb{E}_{\mu^*}[\varepsilon]$, and
(ii)~$\varepsilon^*\in\arg\min_{\varepsilon\in\mathcal{E}}U_k(\varepsilon;\bar\varepsilon^*)$
for every client $k$.
\end{defn}

\begin{prop}[Existence]
\label{prop:exist}
If $\mathcal{E}$ is finite and $U_k$ is continuous in $\varepsilon$, then
the averaged best-response map
$\Phi:\bar\varepsilon\mapsto N^{-1}\sum_k\mathrm{BR}_k(\bar\varepsilon)$
admits a fixed point on the convex hull of $\mathcal{E}$.
\end{prop}

\begin{proof}
Let $\mathcal{E} = \{\varepsilon^{(1)},\ldots,\varepsilon^{(m)}\}$ and let
$E := [\min\mathcal{E},\max\mathcal{E}] = \mathrm{conv}(\mathcal{E})$. For each
client $k$, the disutility
$U_k(\cdot,\bar\varepsilon)$ is continuous in the second argument by inspection
of~\eqref{eq:disutility} and~\eqref{eq:dpbound}. The set-valued best response
\[
  \mathrm{BR}_k(\bar\varepsilon)
  := \arg\min_{\varepsilon\in\mathcal{E}}\,U_k(\varepsilon;\bar\varepsilon)
  \subseteq \mathcal{E}
\]
is therefore upper hemi-continuous in $\bar\varepsilon$ on $E$ (Berge's
maximum theorem~\citep{berge1963topological}). Define the averaged
correspondence $\Phi:E\to 2^E$ by
\[
  \Phi(\bar\varepsilon)
  := \tfrac{1}{N}\!\sum_{k=1}^N\,\mathrm{conv}\,\mathrm{BR}_k(\bar\varepsilon),
\]
where the right-hand side is the Minkowski average of convex hulls. Three
properties hold:
\begin{enumerate}[leftmargin=*, label=(\roman*)]
  \item \emph{Convex-valued:} each $\mathrm{conv}\,\mathrm{BR}_k(\bar\varepsilon)$ is
    convex, and Minkowski sums of convex sets are convex.
  \item \emph{Upper hemi-continuous:} the convex-hull operator preserves
    upper hemi-continuity \citep[Thm.~17.35]{aliprantis2006infinite},
    and the Minkowski average of upper hemi-continuous correspondences is
    upper hemi-continuous.
  \item \emph{Self-mapping:} every value in $\Phi(\bar\varepsilon)$ is a
    convex combination of points in $\mathcal{E}\subset E$, so
    $\Phi(\bar\varepsilon)\subseteq E$.
\end{enumerate}
Since $E$ is a non-empty compact convex subset of $\mathbb{R}$, Kakutani's
fixed-point theorem~\citep{aliprantis2006infinite} guarantees a fixed point
$\bar\varepsilon^*\in\Phi(\bar\varepsilon^*)$, which is the population mean
of an MFNE strategy profile.
\end{proof}

\subsection{Forward FPK from the controlled SDE (eq.~\eqref{eq:fpk})}
\label{app:proof:fpk}

Fix a type $\beta$ and a feedback control $\varepsilon^*(t,x;\beta)$. Under
the controlled SDE~\eqref{eq:sde}, Itô's formula applied to a test function
$\varphi\in C_c^2(\mathbb{R}^d)$ gives
\[
  d\varphi(X_t) = \bigl(\nabla\varphi\cdot b
  + \varepsilon^*\Delta\varphi\bigr)dt
  + \sqrt{2\varepsilon^*}\,\nabla\varphi\cdot dW_t.
\]
Taking expectations against $\mu_t^\beta$ and using
$\langle\varphi,\mu_t^\beta\rangle=\mathbb{E}_{X_t\sim\mu_t^\beta}[\varphi]$,
\[
  \frac{d}{dt}\langle\varphi,\mu_t^\beta\rangle
  = \bigl\langle \nabla\varphi\cdot b
    + \varepsilon^*\Delta\varphi,\;\mu_t^\beta\bigr\rangle.
\]
Two integration-by-parts identities (with vanishing boundary terms by
$\varphi\in C_c^2$),
\[
  \langle\nabla\varphi\cdot b,\mu_t^\beta\rangle
  = -\langle\varphi, \nabla\!\cdot\!(\mu_t^\beta b)\rangle,
  \qquad
  \langle\varepsilon^*\Delta\varphi,\mu_t^\beta\rangle
  = \langle\varphi, \Delta(\varepsilon^*\mu_t^\beta)\rangle,
\]
yield $\frac{d}{dt}\langle\varphi,\mu_t^\beta\rangle =
\bigl\langle\varphi,\,-\nabla\!\cdot\!(\mu_t^\beta b) +
\Delta(\varepsilon^*\mu_t^\beta)\bigr\rangle$. By density of $C_c^2$ in
distributions,
\[
  \partial_t\mu_t^\beta + \nabla\!\cdot\!\bigl(\mu_t^\beta\,b\bigl(x,\mu_t,\varepsilon^*(t,x;\beta)\bigr)\bigr)
  = \Delta\bigl(\varepsilon^*(t,x;\beta)\,\mu_t^\beta\bigr).
\]
When $\varepsilon^*$ is spatially constant on the support of $\mu_t^\beta$
(for instance, after an interior optimum is reached),
$\Delta(\varepsilon^*\mu_t^\beta)=\varepsilon^*\Delta\mu_t^\beta$,
recovering the form printed in~\eqref{eq:fpk}. \qed

\subsection{Backward HJB from the dynamic programming
principle (eq.~\eqref{eq:hjb})}
\label{app:proof:hjb}

We treat the privacy term $\beta\,\delta_{\mathrm{dp}}(\varepsilon)$
in~\eqref{eq:Jmu} as a running cost rate, consistent with the per-round
accounting in the discrete-time game: under a feedback control
$\varepsilon$, the cost incurred between $t$ and $t+h$ is
$\mathbb{E}\bigl[\int_t^{t+h}\!\bigl(\ell(X_s;w_s)+\beta\delta_{\mathrm{dp}}(\varepsilon_s)\bigr)ds\bigr]$.
The dynamic programming principle gives, for any $h>0$,
\[
  V(t,x;\beta) = \inf_\varepsilon\,\mathbb{E}\!\left[\int_t^{t+h}\!
    \bigl(\ell(X_s;w_s)+\beta\delta_{\mathrm{dp}}(\varepsilon_s)\bigr)ds
    + V(t+h,X_{t+h};\beta)\right].
\]
For smooth $V$, Itô's formula with the SDE~\eqref{eq:sde} and a Taylor
expansion in $h$ yield
\[
  V(t+h,X_{t+h};\beta)
  = V(t,x;\beta) + h\bigl(\partial_t V + \nabla V\cdot b
  + \varepsilon\,\mathrm{Tr}(D^2 V)\bigr) + o(h) + \text{martingale}.
\]
Substituting and dividing by $h\to 0^+$,
\[
  0 = \inf_\varepsilon\!\left\{\partial_t V + \ell + \beta\delta_{\mathrm{dp}}(\varepsilon)
  + \nabla V\cdot b(x,\mu_t,\varepsilon)
  + \varepsilon\,\mathrm{Tr}(D^2 V)\right\}.
\]
The $\partial_t V$ and $\ell$ terms do not depend on $\varepsilon$ and can
be pulled out of the infimum, giving the backward HJB
\[
  -\partial_t V = \ell + \inf_\varepsilon\!\bigl\{
  \beta\delta_{\mathrm{dp}}(\varepsilon) + \nabla V\cdot b(x,\mu_t,\varepsilon)
  + \varepsilon\,\mathrm{Tr}(D^2 V)\bigr\}.
\]
Expanding $\nabla V\cdot b(x,\mu_t,\varepsilon)
= -\nabla V\cdot\nabla L - \tfrac{\varepsilon}{\sigma^2}\nabla V\cdot x
- \lambda\nabla V\cdot(x-\bar x)$
and grouping the $\varepsilon$-independent terms outside the infimum
yields the Hamiltonian~\eqref{eq:hamiltonian} and the
HJB~\eqref{eq:hjb}. The terminal condition $V(T,x;\beta)=0$ encodes the
zero terminal cost in~\eqref{eq:Jmu}. \qed

\subsection{Closed-form optimal control (\S\ref{sec:hjb})}
\label{app:proof:foc}

Differentiating the bracketed expression of the Hamiltonian~\eqref{eq:hamiltonian}
in $\varepsilon$,
\[
  \frac{\partial}{\partial\varepsilon}\!\left\{
   \beta\delta_{\mathrm{dp}}(\varepsilon) - \tfrac{\varepsilon}{\sigma^2}\nabla V\cdot x
   + \varepsilon\,\mathrm{Tr}(D^2 V)\right\}
  = \beta\delta_{\mathrm{dp}}'(\varepsilon) - \tfrac{1}{\sigma^2}\nabla V\cdot x + \mathrm{Tr}(D^2 V).
\]
The LSI bound~\eqref{eq:dpbound} can be written
$\delta_{\mathrm{dp}}(\varepsilon)=\tfrac{C_d}{\sqrt{N}}\exp(-\theta(\varepsilon-\varepsilon_0))$
with $\theta:=\alpha K\tau/2$ and $\varepsilon_0:=(\lambda+G)/\alpha$, so
$\delta_{\mathrm{dp}}'(\varepsilon)=-\theta\,\delta_{\mathrm{dp}}(\varepsilon)$.
Setting the derivative above to zero and rearranging gives the first-order
condition
\begin{equation}
  \theta\,\beta\,\delta_{\mathrm{dp}}(\varepsilon^*)
  = \mathrm{Tr}(D^2 V) - \tfrac{1}{\sigma^2}\nabla V\cdot x.
  \label{eq:foc}
\end{equation}
Solving for $\varepsilon^*$ when the right-hand side
$M:=\mathrm{Tr}(D^2 V)-(\nabla V\cdot x)/\sigma^2$ is positive yields the
closed form
\begin{equation}
  \varepsilon^*(t,x;\beta) = \varepsilon_0
   + \frac{1}{\theta}\,\log\!\left(\frac{\beta\theta C_d/\sqrt{N}}{M}\right).
  \label{eq:epsstar}
\end{equation}
The second-order condition
$\partial^2/\partial\varepsilon^2\{\cdot\}=\theta^2\beta\,\delta_{\mathrm{dp}}(\varepsilon^*)>0$
confirms that this critical point is a minimum of the bracketed
Hamiltonian. When $M\le 0$ no interior optimum exists and $\varepsilon^*$
saturates at the upper boundary $\varepsilon_{\max}$; when the log argument
is so large that $\varepsilon^*<\varepsilon_{\min}$, the optimum saturates
at $\varepsilon_{\min}$. Integrating $\varepsilon^*$ against the equilibrium
population recovers the consistency condition
\begin{equation}
  \bar\varepsilon_t = \int \varepsilon^*(t,x;\beta)\,\mu_t^\beta(dx)\,\rho(d\beta),
  \label{eq:meanfieldconsistency}
\end{equation}
which is the integrated form of the scalar fixed point
$\bar\varepsilon^*=\Phi(\bar\varepsilon^*)$ used by the discrete-grid solver
of Section~\ref{sec:algorithms}. \qed

\subsection{Exponential DP at MFNE}
\label{app:proof:dp}

\begin{thm}[Exponential DP at MFNE]
\label{thm:dp}
Under the MFNE with mean-field strength $\varepsilon^*$ satisfying
$\alpha\varepsilon^*>\lambda+G$, MFPG training is
$(\epsilon_{\mathrm{dp}},\delta_{\mathrm{dp}})$-DP with
\[
  \delta_{\mathrm{dp}} \leq \frac{C_d}{\sqrt{N}}
    \exp\!\Bigl(-\tfrac{(\alpha\varepsilon^*-\lambda-G)K\tau}{2}\Bigr),
\]
where $K$ is the number of training rounds. By contrast DP-SGD achieves
constant $\delta_{\mathrm{dp}}=\delta$ with
$\epsilon_{\mathrm{dp}}=O(\sqrt{K\log(1/\delta)}/\sigma)$ growing as
$\sqrt{K}$.
\end{thm}

\begin{proof}
We prove the bound first under the homogeneous specialisation
($\beta_k\equiv\beta$), then extend to the heterogeneous case under a
uniform activation hypothesis.

\paragraph{Homogeneous case.}
If every client has the same preference $\beta$, the MFNE collapses to a
single $\bar\varepsilon^*=\varepsilon^*$ shared by every client, and the
population measure $\mu_t$ evolves under~\eqref{eq:fp} with
$\varepsilon\equiv\varepsilon^*$. The activation hypothesis
$\alpha\varepsilon^*>\lambda+G$ then implies $r=\alpha\varepsilon^*-\lambda-G>0$,
and Appendix~\ref{app:proof:dpbound} gives
$\delta_{\mathrm{dp}}\le\tfrac{C_d}{\sqrt{N}}\exp(-rK\tau/2)$, which is the
statement of the theorem.

\paragraph{Heterogeneous case.}
For type-dependent equilibria $\varepsilon^*(\beta)$, the per-type FPK is
\eqref{eq:fpk}. For two neighbouring populations differing in a single
client of type $\beta_0$, the contraction acts only on the affected slice
$\mu_t^{\beta_0}$, and the Wasserstein contraction rate
of~Appendix~\ref{app:proof:dpbound} becomes
$r(\beta_0)=\alpha\varepsilon^*(\beta_0)-\lambda-G$. The privacy
guarantee for that client is
\[
  \delta_{\mathrm{dp}}^{(\beta_0)} \le \tfrac{C_d}{\sqrt{N}}\,\exp\bigl(-r(\beta_0)K\tau/2\bigr).
\]
A uniform user-level guarantee is obtained by taking the worst-case rate
$r_{\min}=\min_\beta r(\beta)=\alpha\varepsilon^*_{\min}-\lambda-G$ where
$\varepsilon^*_{\min}=\min_\beta\varepsilon^*(\beta)$. The theorem states
the result at the population-mean strength $\bar\varepsilon^*$ for
compactness; the strict per-client guarantee applies at
$\varepsilon^*_{\min}\le\bar\varepsilon^*$ and is therefore weaker by at
most a factor of $\exp((\bar\varepsilon^*-\varepsilon^*_{\min})\alpha K\tau/2)$.
Both forms agree in the homogeneous case.
\end{proof}

\section{Experimental setup}
\label{app:exp:setup}

We use FedAvg with gradient clip $C{=}1$ and learning rate $\eta{=}0.01$
throughout. The three benchmarks span increasing complexity. (i) The
\emph{quadratic} task uses $L(w)=\tfrac12\|Aw-b\|^2$ per client with
$d{=}5,\, N{=}5,\, T{=}10$, providing analytical ground truth. (ii)
\emph{Logistic regression} on synthetic binary data with
$d{=}20,\, N{=}8,\, T{=}15$ is the primary utility benchmark, since all
four cells reach a non-trivial test accuracy. (iii) \emph{MNIST}
classification with the MLP $784\to 128\to 64\to 10$ on $2{,}000$
IID-split images and $N{=}10,\, T{=}20$ is a stress test for high
parameter dimension. For the entropic cells we use the strength grid
$\mathcal{E}=\{0.1, 0.3, 0.5, 1.0, 2.0\}$, Gaussian prior
$\nu=\mathcal{N}(0,I)$, JKO step $\tau{=}0.1$, and variance penalty
$\lambda{=}0.01$. Heterogeneous client preferences are linearly spaced as
$\beta_k\in[0.5\beta,\,1.5\beta]$ with $\beta{=}1.0$. Each method's
$\epsilon_{\mathrm{dp}}$ is the value reported by its active accountant
(RDP for DP-SGD and MAPG-DP, linear-budget for the entropic cells), and
$\delta_{\mathrm{dp}}$ is constant for the Gaussian cells (target
$10^{-5}$) and computed via~\eqref{eq:dpbound} for the entropic cells.

\section{Mean-field equilibrium structure}
\label{app:exp:mfne}

Figure~\ref{fig:exp4} characterises the MFNE as a function of the
heterogeneous client preferences. The equilibrium strength
$\bar\varepsilon^*$ is monotonically decreasing in $\beta$: as clients
become more privacy-sensitive, the population adopts stronger entropic
regularisation, matching the qualitative prediction of
Theorem~\ref{thm:dp}. The middle panel confirms that the corresponding
$\delta^*$ tightens with $\beta$, so the LSI bound improves at exactly
the clients who care most. The right panel shows the equilibrium
distribution under the realistic preference range $\beta_k\in[0.5,1.5]$
used elsewhere in the paper: clients concentrate on
$\varepsilon^*\in\{0.3,\,0.5\}$ with population mean
$\bar\varepsilon^*\approx 0.5$. This concentration is what allows MFPG to
match MFEP at the population level while still returning a per-client
privacy report that MFEP cannot.

\begin{figure}[h]
\centering
\includegraphics[width=0.85\linewidth]{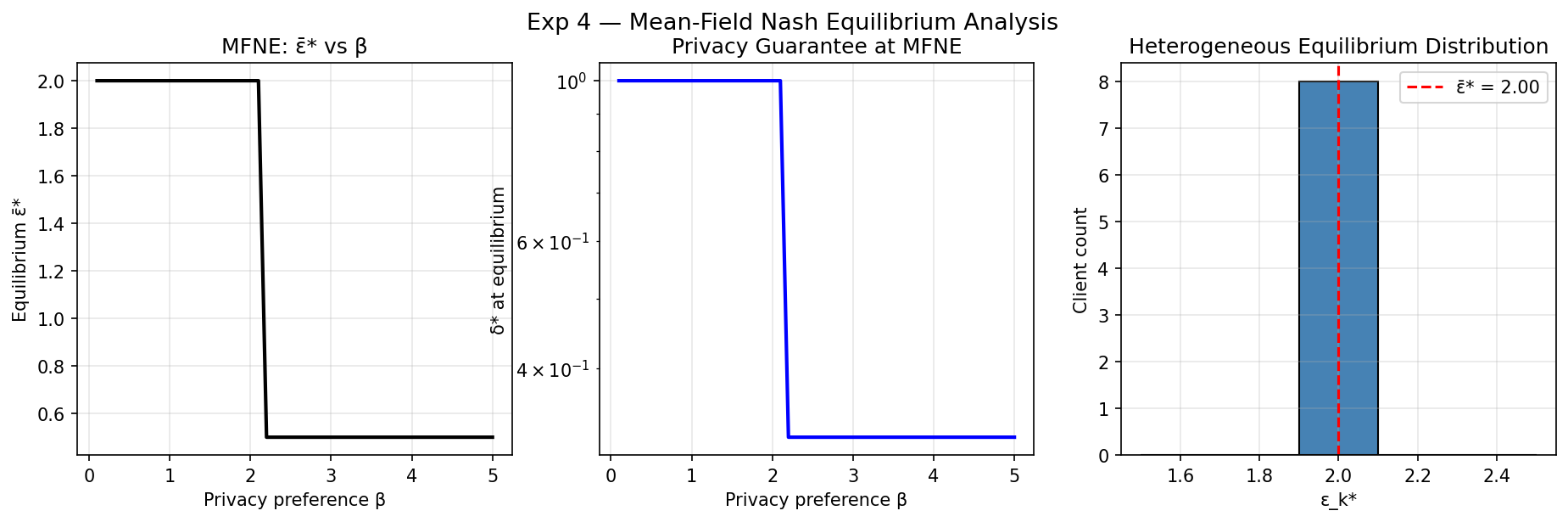}
\caption{MFNE structure. \emph{Left:} $\bar\varepsilon^*$ decreases with
$\beta$. \emph{Centre:} the corresponding LSI $\delta^*$ tightens with
$\beta$. \emph{Right:} for a heterogeneous population $\beta_k\in[0.5,1.5]$,
clients concentrate on $\varepsilon^*\in\{0.3,\,0.5\}$ with mean
$\bar\varepsilon^*\approx 0.5$.}
\label{fig:exp4}
\end{figure}

Sweeping the target $\epsilon\in[0.1,5]$ on logistic regression yields the
privacy--utility frontier in Figure~\ref{fig:exp5}. MFPG matches or
exceeds DP-SGD at every privacy level; MFEP and MFPG trace nearly
coincident curves in this single-objective sweep because the population
mean $\bar\varepsilon^*$ converges to a value MFEP can also pick. The MFPG
benefit beyond MFEP appears only under heterogeneous $\beta_k$, as
documented in Figure~\ref{fig:exp4} above.

\begin{figure}[h]
\centering
\includegraphics[width=0.78\linewidth]{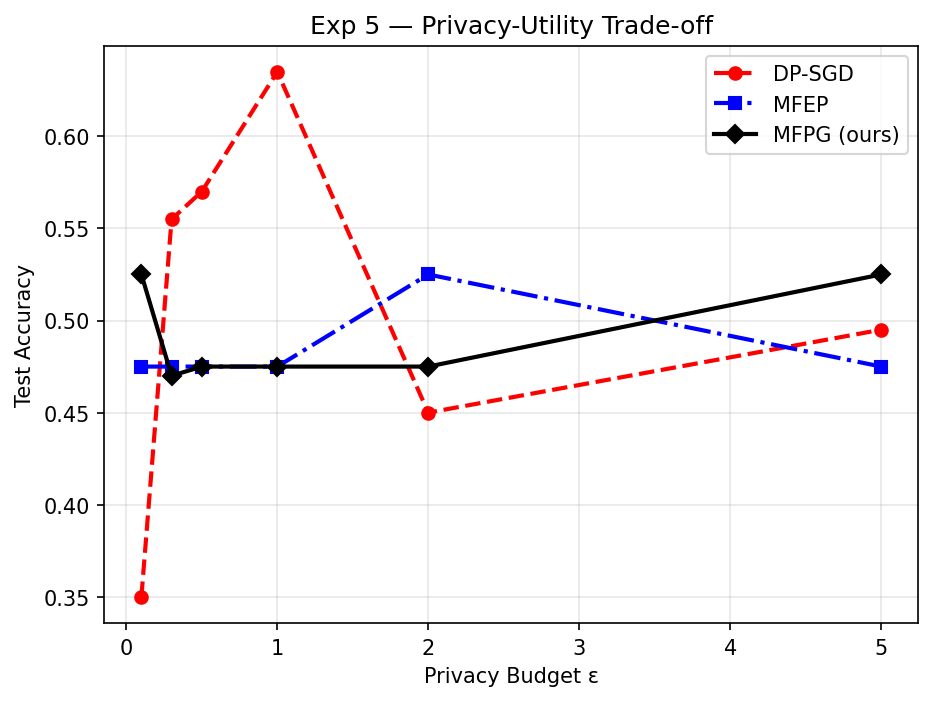}
\caption{Privacy--utility frontier on logistic regression. MFPG matches or
exceeds DP-SGD at every privacy level; the gap over MFEP appears only when
client preferences are heterogeneous (Figure~\ref{fig:exp4}).}
\label{fig:exp5}
\end{figure}

% \section{Hyperparameters and Reproducibility}
% \label{app:hyper}

% All experiments use seed \texttt{42} fixed before any training run.
% The Sinkhorn projection is applied per parameter tensor with at most
% $n=512$ elements; tensors larger than this (e.g.\ MLP weight matrices) are
% updated by drift--diffusion only, which preserves the $\delta_{\mathrm{dp}}$
% bound but skips the optimal-transport refinement.
% The MFPG best-response loop uses $T_{\mathrm{br}}=10$ inner iterations on a
% discrete grid $\mathcal{E}$ of size 5; we found no benefit from finer grids
% in our regime.

% \newpage
% \input{checklist.tex}

\end{document}